%% file: main-aaai-2022.tex
\relax

\documentclass[letterpaper]{article} 
\usepackage{aaai22}  
\usepackage{times}  
\usepackage{helvet}  
\usepackage{courier}  
\usepackage[hyphens]{url}  
\usepackage{graphicx} 
\usepackage{natbib}  
\usepackage{caption} 
\DeclareCaptionStyle{ruled}{labelfont=normalfont,labelsep=colon,strut=off} 
\usepackage{algorithm}
\usepackage{algorithmic}

\usepackage{newfloat}
\usepackage{listings}
\floatstyle{ruled}
\newfloat{listing}{tb}{lst}{}
\floatname{listing}{Listing}
\usepackage[utf8]{inputenc} 
\usepackage[T1]{fontenc}    
\usepackage{hyperref}       
\usepackage{url}            
\usepackage{booktabs}       
\usepackage{amsfonts}       
\usepackage{nicefrac}       
\usepackage{microtype}      
\usepackage{xcolor}         
\usepackage{amsmath}
\usepackage{natbib}
\usepackage{algorithm}
\usepackage{algorithmic}
\usepackage{mathtools}
\usepackage{array}

\DeclareMathOperator*{\argmin}{arg\,min} 
\usepackage{amsthm}
\newtheorem{thm}{Theorem}[section]

\newtheorem{theorem}{Theorem}[section]

\newtheorem{definition}{Definition}

\newtheorem{lemma}[theorem]{Lemma}

\newcommand{\norm}[1]{\left\lVert#1\right\rVert}

\newcommand{\myparagraph}[1]{\vspace*{-2mm}\paragraph{ #1}}

\newcommand{\eps}{\epsilon}
\DeclareMathOperator{\Lap}{Lap}
\newcommand{\mb}{\mathbf}

\newcommand{\CASE}[1]{\STATE \textbf{case} #1\textbf{:} \begin{ALC@g}}
\newcommand{\ENDCASE}{\end{ALC@g}}

\newcommand{\DEFAULT}{\STATE \textbf{default:} \begin{ALC@g}}
\newcommand{\ENDDEFAULT}{\end{ALC@g}}
\newcommand{\DEFAULTLINE}[1]{\STATE \textbf{default:} }

\title{Continual and Sliding Window Release for Private Empirical Risk Minimization}

\author{
    Lauren Watson\textsuperscript{\rm 1},
    Abhirup Ghosh\textsuperscript{\rm 2},
    Benedek Rozemberczki\textsuperscript{\rm 1},
    Rik Sarkar\textsuperscript{\rm 1}
}
\affiliations{
    \textsuperscript{\rm 1}University of Edinburgh, \textsuperscript{\rm 2} University of Cambridge\\

    lauren.watson@ed.ac.uk
}

\begin{document}

\maketitle

\begin{abstract}
It is difficult to continually update private machine learning models with new data while  maintaining privacy. Data incur increasing privacy loss -- as measured by differential privacy -- when they are used in repeated computations. In this paper, we describe regularized empirical risk minimization algorithms that continually release models for a recent window of data. One version of the algorithm uses the entire data history to improve the model for the recent window. The second version uses a sliding window of constant size to improve the model, ensuring more relevant models in case of evolving data. The algorithms operate in the framework of stochastic gradient descent. We prove that even with releasing a model at each time-step over an infinite time horizon, the privacy cost of any data point is bounded by a constant  $\epsilon$ differential privacy, and the accuracy of the output models are close to optimal. Experiments on MNIST and Arxiv publications data show results consistent with the theory. 
\end{abstract}

\input{sections/intro}
\input{sections/relatedwork}
\input{sections/background}
\input{sections/problem}

\input{sections/basic}

\input{sections/continual2}
\input{sections/slidingwindow2}

\input{sections/variants}
\input{sections/experiments}
\input{sections/conclusion}



\newpage
\bibliography{sample-base}

\newpage
\appendix

\input{sections/appendices}


\end{document}

%% file: sections/intro.tex
\section{Introduction}

Differential privacy (often abbreviated as DP) is a rigorous mathematical definition that characterizes the privacy properties of algorithms. It measures the information that can be learned about an individual by observing the output of the algorithm. The definition is quite general, and since its introduction in~\citet{dwork2006calibrating}, differential privacy has become the standard privacy notion in many areas. It is currently used by major online service providers and organizations including the United States Census Bureau~\cite{abowd2018us}. 

Differentially private machine learning is commonly used to protect personal privacy in analytics of data such as in medical, social and online applications~\cite{kim2018privacy, task2012guide}. In these scenarios, data often evolve, and new machine learning models must be computed incorporating the new data. However, model recomputation poses a challenge -- it degrades privacy as more information is leaked along with the new model. In this paper, we consider the problem of releasing updated ML models while limiting privacy loss. 

Our aim is to release models based on the ``new'' data, while also leveraging historical information. There are multiple reasons for this approach. First, we do not wish training sets to be disjoint, but instead want {\em sliding window} models, e.g. those that are accurate on data for the {\em last $n$ days} -- reflecting recent trends in data. Second, the most recent block of data may be insufficient for training complex models, we wish to leverage historical data to enhance and stabilize the model. Unfortunately, reusing historical data in this context poses a challenge to maintaining differential privacy. 

Previous works on differential privacy for evolving datasets have considered questions of releasing a model for the entire cumulative data. In ~\citet{chan2011private,dwork2010differential}, specific simple statistics such as a count are released at constant intervals of time. The work in~\citet{cummings2018differential} can operate with more general queries such as linear histogram queries and empirical risk minimization (ERM), but publishes models at exponentially growing intervals. In contrast to these approaches, our priority is to maintain up-to-date private models with accuracy guarantees for ``recent data''. 

In our approach, models are published at constant time intervals, but can be applied to general computations such as linear queries over histograms and ERM;  we focus particularly on the important case of regularized ERM. At time $t$, the models make use of historical data -- either the entire historical data, or most recent $w$ items. The models are shown to maintain guaranteed accuracy on the most recent data, and good accuracy on the historic data where it is still relevant. We show that even with constant interval release over an unbounded time horizon, the privacy leak in our method is bounded by a constant $\eps$-DP.


\myparagraph{Our contributions.} Our objective is to periodically release models that are accurate on a recent {\em target} data window of size $b_0$, while making use of a larger {\em source} data window $w$. This problem can be viewed as domain adaptation or transfer learning (see~\citet{redko2020survey} for a survey on transfer learning). In a scenario unconcerned about privacy, simple transfer learning can be applied repeatedly to achieve these models; our contribution is an algorithm that ensures differential privacy through repeated data use in a sliding window. 

In the section \textit{Continual Cumulative Updates} we describe algorithms for the case where the entire historical dataset is used as the training window $w$. As more data accumulates, we maintain differentially private {\em base} models on the growing historical data. These base models are updated regularly to get source models, and at longer intervals they are recomputed from scratch. We prove bounds on the performance of the model on both the new and old data. We show that the privacy loss is bounded by $\eps$-DP even when data is reused over an infinite time horizon. 

The question of a fixed-size (or sliding-window) of historical data $w$ is considered in the \textit{Sliding window model release} section. This problem arises when the data distribution changes over time, so that very old data is no longer representative. Thus, our main challenge is to update models to ``forget'' old data. Our sliding window algorithm maintains a hierarchy  of updates to a base model that are removed and added appropriately with the sliding window, and show privacy and utility bounds as above. 

In the experimental results, we show the performance for logistic regression on two datasets -- MNIST and Arxiv publication records. Empirical results show that the proposed algorithms demonstrate the expected trade-offs between privacy and utility for varying levels of regularization and datasize. Accuracy of the private models approaches non-private performance for strong privacy ($\epsilon=1$) given appropriate levels of regularization and minimum update sizes. 

For the sake of simplicity, we have focused the discussion on a single data stream which contains both the source $w$ and target $b_0$. But in the usual transfer learning setup, these methods can be used for {\em private dynamic transfer learning} -- where $w$ is taken in the source domain which has more data, while $b_0$ is taken in a target domain which is sparser in data, and both datasets evolve with time. 

%% file: sections/relatedwork.tex
\section{Related Work}

Differentially private machine learning has been a major topic in recent years; see~\citet{ji2014differential} for a survey. The most prominent setting is the static database, with the objective of empirical risk minimization with differential privacy. Differentially private logistic regression was described in~\citet{chaudhuri2009privacy} and extended to general regularized empirical risk minimization by~\citet{chaudhuri2011differentially}. 
Further analyses of DP-ERM for strongly convex functions were described in~\citet{kifer2012private,wang2017} and Algorithm 3 of~\citet{bassily2014private}. 

Evolving databases have been considered in different forms.  Online learning is a setup where the evolving input and loss functions may be adversarial, and the objective is to minimize {\em regret}, which measures overall error. Differentially private versions of online learning have been studied in~\citet{jain2013differentially,guha2013nearly} and~\citet{agarwal2017price}. A recent work has considered batch online DP learning minimizing  regret and excess population risk~\cite{pmlrv139kairouz21b}. In contrast, our setup involves independently evolving data, a sliding wondow and the objective is to minimize the empirical error. 

{\em Continual release} of query answers on the entire dataset was described in \citet{dwork2010differential} and~\citet{chan2011private} for a restricted input type. They consider a single bit of input at every round (such as a single element possibly being added to a the set) and publish a differentially private count in every round. A question of adaptive analysis on evolving data was considered in \citet{cummings2018differential}. Building on previous works~\citep{blum2013learning,hardt2010multiplicative} on adaptive queries, \citet{cummings2018differential} describe releasing responses to a broader class of queries. Instead of every round, these results are published at exponentially growing intervals. Our algorithm instead provides updated algorithms are constant-sized ($b_0$) intervals, which is a significant step towards full continual release. Private matrix analysis in the sliding window model has been studied by \citet{upadhyay2020framework}. 

Transfer learning or domain adaptation has been studied from various perspectives (\cite{david2010impossibility,mansour2009domain}). A recent survey can be found in~\cite{redko2020survey}. A variant called {\em hypothesis transfer learning}~\citep{kuzborskij2013stability,mansour2008domain} is particularly relevant for us. In this variant, once a suitable hypothesis (model) on the source domain is computed, the source data is no longer accessed, and only the hypothesis (model) is used for learning in the target domain. 

%% file: sections/background.tex

\section{Preliminaries}

\textbf{Empirical risk minimization} Let $D=\{(\mathbf{x}_i, y_i) : i \in [1, n]\}$ represent a dataset of training examples drawn from some underlying distribution $\mathcal{D}\sim\mathcal{X} \times \mathcal{Y}$. Consider a set of candidate models $\mathcal{F}$ where $f: \mathcal{X} \rightarrow \mathcal{Y}$ for $f \in \mathcal{F}$ and $f$ is parameterized by weights $\mathbf{w}\in \mathcal{W}$ e.g. $\mathcal{W}=\mathbb{R}^d$. The loss of model $f\in\mathcal{F}$ for each data point is given by a loss function $\ell: \mathcal{Y} \times\mathcal{Y} \rightarrow \mathbb{R}$. Empirical risk minimization describes the learning paradigm of selecting the predictor $\hat{f}\in \mathcal{F}$, parameterized by weights $\mathbf{w}_{\hat{f}}$ such that, 
\begin{equation}
    \hat{f}=\argmin_{f\in\mathcal{F}} \hat{L}_D(\mathbf{w_f}) = \argmin_{f\in \mathcal{F}}  \frac{1}{n} \sum_{i=1}^n \ell(f(\mathbf{x_i}), y_i)
    \label{eq:erm}
\end{equation}
where $\hat{L}_D(\mathbf{w})$ represents the empirical risk; and the true risk is denoted by $L(\mathbf{w})=\mathbb{E}_{D\sim \mathcal{D}}[\ell(f(\mathbf{x_i}), y_i)]$. Regularized empirical risk minimization with regularization parameter $\lambda>0$ includes a penalization term in the loss function of the form $\ell(f(\mathbf{x_i}), y_i) \leftarrow \ell(f(\mathbf{x_i}), y_i) + \lambda \lVert\mathbf{w}_f\rVert_2^2$, ensuring the $\lambda$-strong convexity of the loss function.The following assumptions are commonly made about the loss function $\ell$:
\begin{definition}[Convexity]
  A function $f:\mathcal{X}\rightarrow\mathcal{S}$ satisfies $\lambda$-strong convexity if for all $x, y\in \mathcal{X}$:
\[
    f(x) \geq f(y) + \nabla f(x)^T(x-y) + \frac{\lambda}{2}\lVert x-y \rVert_2^2
\]
\end{definition}
\begin{definition}[Lipschitzness] 
  A function $f:\mathcal{X}\rightarrow\mathcal{S}$ is $L$-Lipschitz if for any $x, y\in \mathcal{X}$: $\lVert f(x)-f(y)\rVert_2\leq L\lVert x-y \rVert_2$
\end{definition}
\begin{definition}[smoothness]
  A differentiable function $f:\mathcal{X}\rightarrow\mathcal{S}$ is $\beta$-smooth if for ant $x, y\in \mathcal{X}$:
$\lVert \nabla f(x)- \nabla f(y)\rVert_2 \leq \beta\lVert x-y\rVert_2$


\end{definition}
Differential privacy is a rigorous mathematical definition of privacy for a randomized algorithm $A$. It is defined in terms of a pair of neighboring databases $(D, D')$: Two databases $D, D'$ are neighboring if $H(D, D') \leq 1$, where $H(\cdot, \cdot)$ represents the Hamming distance.  This notion corresponds to event-level privacy: where the presence or absence of single datapoints is obscured.

\begin{definition}[Differential privacy]
A randomized algorithm $A$ satisfies $\epsilon$-differential privacy if for all neighboring databases $D,D'$ and for all possible outputs $O\subseteq \text{Range}(A)$,
  \[ \Pr[A(D) \in O] \leq e^{\epsilon} \cdot \Pr[A(D') \in O].\]
\end{definition}

The L2 sensitivity of a function is the maximum change in the function value between neighboring databases: $\Delta f = \max_{D, D' \in \mathcal{D}} \lvert f(D)-f(D') \rvert_2$. The sensitivity is used to determine the noise added by $A$ to achieve differential privacy, e.g. the Laplace mechanism works as follows: 

\begin{definition}[Laplace mechanism]
Given any function $f:\mathcal{D}\rightarrow O$ and privacy parameter $\epsilon$, for any $D\in \mathcal{D}$, the Laplace Mechanism returns:
$f(D)+\nu,$ where $\nu \sim \Lap\left(\frac{\Delta f}{\epsilon}\right)$. 
\end{definition}
$\Lap(b)$ is the Laplace distribution with mean $0$ and scale $b$. It is known that this mechanism preserves differential privacy~\cite{dwork2006calibrating} (abbreviated as $\eps$-DP), and that if a sequence of randomized algorithms $A_i$ are applied on a dataset, each with $\eps_i$-DP, then the sequence satisfies $\left(\sum_i \eps_i\right)$-DP. More advanced composition theorems also exist~\cite{kairouz2015composition}. 



%% file: sections/problem.tex
\subsection{Problem description}\label{sec:problem}
Suppose data points $(\mathbf{x}, y)\in \mathcal{X}\times \mathcal{Y}$ arrive in a continual manner and dataset $D=\{(\mathbf{x_i}, y_i) : i\in [0,t]\}$ represents all points that have arrived by time $t$. 
The dataset $D_{[t_1: t_2]} = \{(\mathbf{x_i}, y_i) : i \in [t_1, t_2], 0\leq t_1 \leq t_2 \leq t\}$ represents the set of data points that arrived in the time interval $[t_1, t_2]$, with $D_{[t_1: t_2]}\subseteq D_{[0:t]}$.

Our goal is to compute an ERM model for the most recent batch $b_0$ of data i.e. $D_{[t-b_0: t]}$. We assume that it is sufficient to release a model at constant intervals representing a datasize of at least $b_0$, that is at each $t$ that is a multiple of $b_0$. Further, we assume that a minimum block size $B\geq b_0$ is necessary to compute accurate models. We wish to use information from a historical data window $D_{[t-w: t]}$ of size $w>b_0$ to augment the training on $D_{[t-b_0: t]}$.

We consider two versions: cumulative continual updates (with window $ D_{[0, t]}$) and sliding window $D_{[t-w: t]}$ for some constant $w$. In both these cases, our objective is to obtain an accurate model for $D_{t-b_0, t}$, with an additional objective that the model should also have high accuracy for $D_{t-w, t}$. Note that this may not be possible if the distributions of the source window $W$ and target $b_0$ are significantly different. We also require a differential privacy guarantee protecting the presence or absence of a single data sample in the dataset (event-level privacy as per~\citet{chan2011private}). A constant differential privacy guarantee must be ensured in all cases.

We assume that data arrives at a constant rate. In scenarios where this is not the case, the results can be applied by defining dynamic time steps as the interval that accumulates a certain constant number of data points.

%% file: sections/basic.tex
\subsection{Basic approaches and multi-resolution release}
\label{sec:basic}

Before describing the main results, we briefly discuss some basic and existing approaches that are relevant. 
Consider a simple scenario where in each time-step a single element may arrive. The sensitivity of the {\em count} function is $1$ -- as a single element changes it by $1$. In each time step $t$, the count can be released with a $\Lap{(1/\eps)}$ noise, this guarantees $\eps$-DP for the release at time $t$, however, data that arrives in early rounds suffer information leak in each round. By the composition properties of differential privacy, data that contributes to $T$ releases of the count, has $T\eps$-DP. To ensure $\eps$-DP for each element, the noise scale must grow with time, requiring $\Lap{(T/\eps)}$ noise at round $T$ -- which is excessive for most purposes. 


Certain important classes of functions have smaller sensitivity that yields more efficient algorithms. Linear queries over histograms was considered in~\cite{cummings2018differential,hardt2010multiplicative,blum2013learning} etc. Suppose $U$ is a finite data universe of size $N$, and $D\in U^n$ is a database. We write it as a histogram $x$, where $x^i$ as the fraction of $x$ of type $i\in [N]$, that is: $x^i=n_i/n$. A linear query is described by a vector $f\in [0,1]^N$, where the objective is to return $\langle f,x \rangle$. In this format, the presence or absence of a single element changes the histogram entry for the type by at most $1/n$, and thus has sensitivity of $1/n$. 

In a machine learning context, similar low sensitivity arises for $\lambda$-strongly convex loss functions and regularized ERM (Eq.\ref{eq:erm}). In this case, it can be shown that the sensitivity is bounded by $O\left(\frac{1}{\lambda n}\right)$~\citep{chaudhuri2011differentially}. 
This result is particularly significant, since many important machine learning methods, including regularized convex Stochastic Gradient Descent fit this mold~\citep{wu2017bolt}. An algorithm for private SGD is shown as Algorithm~\ref{alg:output}. For such a mechanism, where the sensitivity is $O(1/n)$, if one ERM model is released in every round with a $\Lap{(1/\eps)}$ noise, then the information leak is bounded by $\Theta((\log n)\eps)$-DP.

\begin{algorithm}
\caption{Private SGD via Output Perturbation (PSGD)~\cite{ wu2017bolt}}
\begin{algorithmic}[1]
  \scriptsize
  \STATE Input: $D=\{(\mathbf{x}_t, y_t) \}$, inverse learning rate $\gamma$, sensitivity $\Delta_{\epsilon}$, number of iterations $m$.
  \STATE  $\mathbf{w} \leftarrow SGD(D)$ with $k$ passes and learning rate $ \frac{1}{\gamma i}$ for iteration $i$.
  \RETURN $\mathbf{w}+\nu$ where $\nu\overset{d}{\sim} Lap\left(\Delta_{\epsilon}\right)$.
\end{algorithmic}
  \label{alg:output}
\end{algorithm}

\myparagraph{Hierarchical and multiresolution release approach.} As a warm-up, we describe a mechanism to release models with a constant $\eps$-differential privacy guarantee. This approach does not achieve the continual release property of producing an output at each round trained using all previous data, but releases these global outputs at exponentially growing intervals (such as in~\citet{cummings2018differential}).

In this approach, a query result (such as a model trained via as shown in Algorithm~\ref{alg:output}) is released at times $t=2^kB$ for $k=0,1,2,3,\dots$. This approach is capable of more than releasing the simple cumulative model. By using a binary hierarchy, we can release models at $\log t$ different scales of sizes $qB, q \in \mathbb{Z}^{+}$. The approach is simply the following: At any time $t=(2^k)q B$ (for $q, k \in \mathbb{Z}^{+}$), we release a model computed on data in the interval $[t-2^kB: t]$. Algorithm~\ref{alg:multires} shows the idea. 

\begin{algorithm}
\caption{Private Multi-Resolution Release}
\begin{algorithmic}[1]
  \scriptsize
  \STATE Input: $D$, $B$, $\Delta_{\eps}=\frac{4L}{\lambda B \epsilon}$ where $L$ is the Lipschitz constant of $\ell$.
   \FOR{$t=2^kqB$ for $q, k \in\mathbb{Z}^+$} 
        \STATE $\mathbf{w^t} \leftarrow PSGD(D=\{(\mathbf{x}_j, y_j) | j\in [t-2^kB+1, t]\}, \Delta_{\epsilon})$ (Algorithm~\ref{alg:output})
        \STATE Release $\mathbf{w^t}$
   \ENDFOR
\end{algorithmic}
  \label{alg:multires}
\end{algorithm}

\begin{figure}[h]
    \centering
    \includegraphics[width=0.85\linewidth]{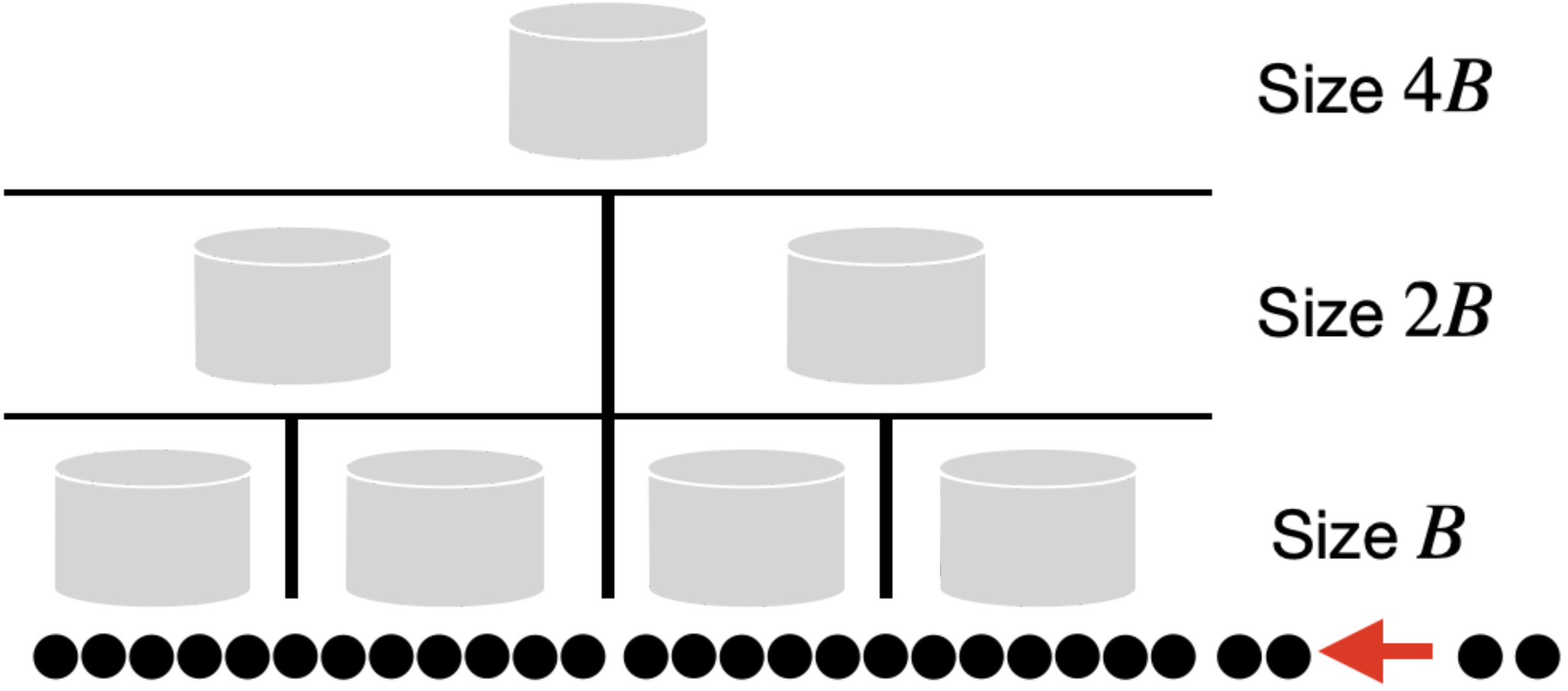}
    \caption{Incoming datapoints are partitioned by Algorithm~\ref{alg:multires} to form a hierarchy of models with datasizes $2^kB$. Tailored noise addition results in a constant privacy guarantee for any datapoint. }
    \label{fig:basicdiagram}
\end{figure}


\begin{thm} \label{thm:basicprivacy}
Algorithm~\ref{alg:multires} satisfies $\epsilon$-differential privacy.
\end{thm} 

Let us now consider the utility achieved by Algorithm~\ref{alg:multires}. Denote the output of a non-private stochastic gradient descent algorithm by $\mathbf{w}$, this corresponds to line 2 in Algorithm~\ref{alg:output}. Let $\mathbf{w}^*$ represent the non-private minimizer $\mathbf{w}^*=\argmin_{\mathbf{w}}\hat{L}_D(\mathbf{w})$. Denote the corresponding private released parameters by $\mathbf{w}_{priv}$. In other words, $\mathbf{w}_{priv}$ is the privatized set of parameters returned in line $3$ of Algorithm~\ref{alg:output}. Theorem~\ref{thm:basicutility} describes the utility of Algorithm~\ref{alg:multires}.

\begin{thm}[\textbf{\cite{wu2017bolt}}] . 
Consider 1-pass private stochastic gradient descent via Algorithm~\ref{alg:output} for $\beta$-smooth and $L$-Lipschitz loss function $\ell$. Suppose $sup_{\mathbf{w}\in\mathcal{W}}\|\ell'(\mathbf{w})\|\leq G$, $\norm{\mathbf{x}}_1\leq1$, $\mathcal{W}$ has diameter $R$ and $\mathbf{w}\in\mathbb{R}^d$. Then, for a dataset of size $2^kB$ with $k\in[0, \lfloor log_2(\frac{t}{B})\rfloor]$
\begin{align*}
    \mathbb{E}[\hat{L}_D(\mathbf{w}_{priv})-\hat{L}_D(\mathbf{w}^*)]\leq \\
    \frac{((L+\beta R^2)+G^2)\log(2^kB)}{\lambda 2^kB}+\frac{4dG^2}{\epsilon \lambda B}
\end{align*}
    
\label{thm:basicutility}
\end{thm}

While this produces a model based on the most recent data at different scales, this approach does not quite give us the continual or sliding window release we require, and instead releases results in windows of varying sizes. It also assumes a large minimum block size of $B$ that is sufficient for accurate training. It does not produce updated models for smaller blocks of size $b_0$ and in many cases does not make use of historical data.

%% file: sections/continual2.tex
\section{Continual cumulative updates}
\label{sec:continualrelease}
To release a model at constant intervals using the cumulative data, we adapt the previous hierarchical mechanism. Our approach is: at any time, we have a {\em base} model $\mb{w}_g$, which has been computed on an early sequence of data up to time $t_g$, in the regular offline manner, with added privacy (e.g. using Algorithm~\ref{alg:output}). As new data arrive, we update $\mb{w}_g$. The update itself proceeds in a hierarchical manner. After a suitable interval, e.g. when $t = 2t_g$, the process is reset, and the base model is recomputed on the entire dataset. Let us now consider the details of the scheme. 

Upon receiving $b_0$ items after $t_g$, i.e., at time $t = t_g + b_0$, the model $f^{t_g}$ is updated to $f^{t}$ using  the data $D_{[t_g+1: t]}$. This update itself is a regularized optimization: 

\begin{equation}\label{eq:regerm}
    f^t= \argmin_{\mb{w}\in\mathcal{W}} \frac{1}{b_0}\sum_{i\in[t_g+1, t]} \ell(f(\mathbf{x}_i),y_i) + \lambda \norm{\mathbf{w} - \mathbf{w}_g}_2^2
\end{equation}
This is solved via a private algorithm shown as  Algorithm~\ref{alg:pb-erm}, which can be realized as an SGD algorithm such as Algorithm~\ref{alg:output}. 

This optimization treats the base model $\mb{w}_g$ as a regularization point, or {\em origin}, and thus the output model $f^t$ stays close to this global model. Further updates at time $t^+ = t+b_0$ can be done with respect to $f^t$ as: 
    $f^{t^+}= \argmin_{\mb{w}\in\mathcal{W}} \frac{1}{b_0}\sum_{i\in[t+1,t+b_0]} \ell(f(\mathbf{x}_i),y_i) + \lambda \norm{\mathbf{w} - \mathbf{w}_t}_2^2$.
A new base model is computed at $t=2^kB$ $D_{[t-w: t]}$ (Algorithm~\ref{alg:privateHTL}).

\begin{algorithm}[H]
\caption{Private Biased Regularized ERM (PBERM)}
\begin{algorithmic}[1]
  \scriptsize
  \STATE Input: Base model $f_g$ with weights $\mathbf{w}_g$, $D_{[t_g+1: t]}$, $\lambda$, $\epsilon$, and $L$.
    \STATE{$\mathbf{w}^t=\argmin_{w\in\mathcal{W}} \frac{1}{b_0}\sum_{i\in[t_g+1, t]} \ell(f(\mathbf{x}_i),y_i) + \lambda \norm{\mathbf{w} - \mathbf{w}_g}^2$}
    \RETURN{$\mathbf{w}^t+\nu$ where $\nu\overset{d}{\sim} Lap\left(\frac{4L}{\lambda b_0 \epsilon}\right)$}
\end{algorithmic}
  \label{alg:privateERM}
  \label{alg:pb-erm}
\end{algorithm}


\begin{algorithm}[H]
\caption{Private Continual Release}
\begin{algorithmic}[1]
  \scriptsize
  \STATE Input: $D$, $\lambda$, $\epsilon$, $L$, $b_0$, $B$.
  \FOR{$t \in \mathbb{Z}^+$ and $t\ge B$}
        \IF {$t=2^{k}B$ for $k\in \mathbb{Z}^{+}$}
            \STATE{Learn base model $f_g$ on $D_{[0:t]}$ using Algorithm~\ref{alg:multires}.}
            \STATE{Release $f_g$ and save $f_c = f_g$ and $t_g = t$}
        \ELSIF{$t-t_g=i b_0$ for $i\in \mathbb{Z}^{+}$}
            \IF{$t - t_g=2^{j}b_0$ for $j\in \mathbb{Z}^{+}$}
                \STATE{Learn $f^t$ on $D_{[t_g: t]}$ and regularize with model $f_g$ using Algorithm~\ref{alg:pb-erm}}
                \STATE{Release $f^t$ and save $f_c = f^t$}
            \ELSE
                \STATE{Learn $f^t$ on $D_{[t-b_0: t]}$, regularize with $f_c$ using Algorithm~\ref{alg:pb-erm}}
                \STATE{Release $f^t$}
            \ENDIF
        \ENDIF
  \ENDFOR
\end{algorithmic}
  \label{alg:privateHTL}
\end{algorithm}

\begin{figure}[h]
    \centering
    \includegraphics[width=0.85\linewidth]{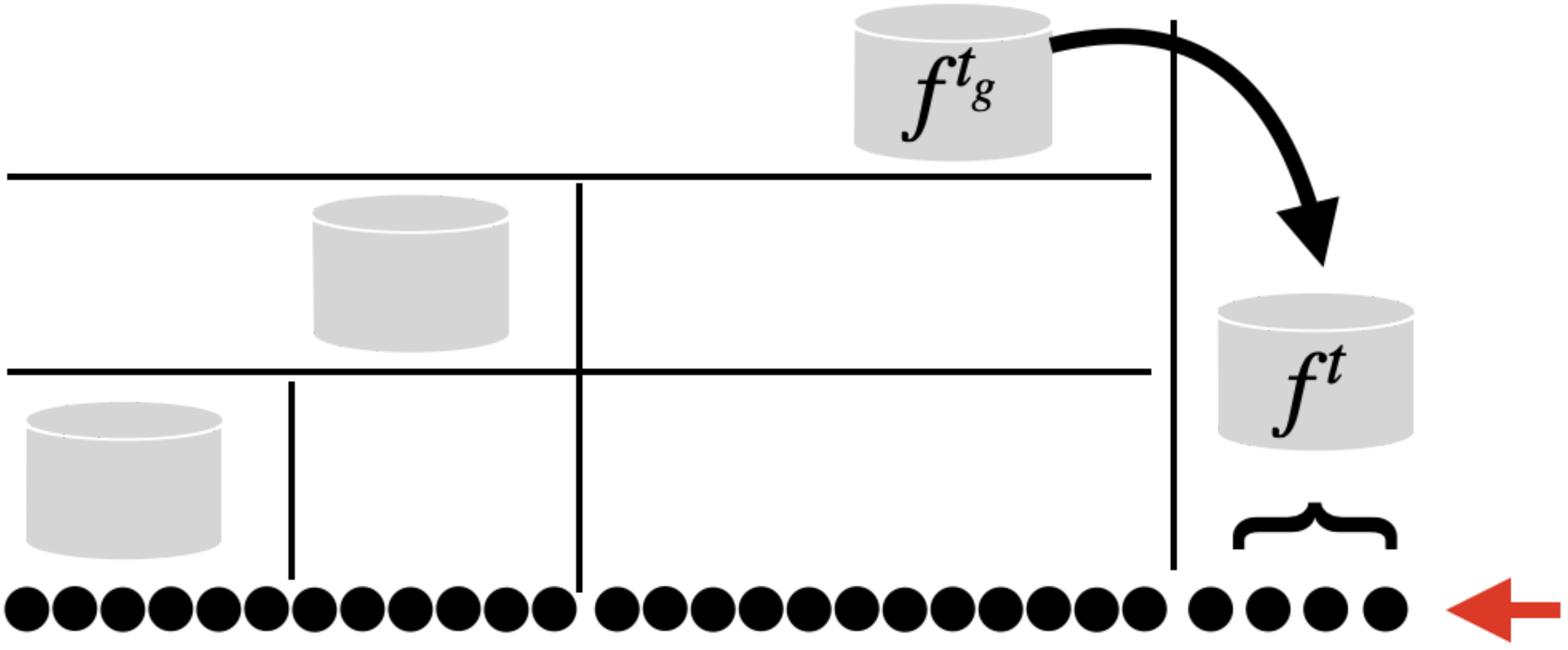}
    \caption{ Algorithm~\ref{alg:privateHTL} updates recent global models using newly arrived data. }
    \label{fig:basicdiagram}
\end{figure}

\myparagraph{Privacy and Utility Guarantees}

\begin{thm}[\citep{wu2017bolt}]\label{thm:continualpriv}
Algorithm~\ref{alg:privateERM} satisfies $\epsilon$-differential privacy.
\end{thm}

Theorem~\ref{thm:continualpriv2} demonstrates that Algorithm~\ref{alg:privateHTL} allows the continual release of models over the incoming dataset with only a constant increase in privacy loss. When used together, the total privacy loss of Algorithms~\ref{alg:multires} and~\ref{alg:privateHTL} is then at most $2\epsilon$.

\begin{thm}\label{thm:continualpriv2}
Algorithm~\ref{alg:privateHTL} satisfies $\epsilon$-differential privacy.
\end{thm}

Theorem~\ref{thm:continual_utility} bounds the excess empirical risk of the continual cumulative models released via Algorithm~\ref{alg:privateHTL}. We denote the model released by $\mathbf{w}_{new}$ and the new data by $D^{new}\sim\mathcal{D}^{new}$. The risk for the true minimizer $\mathbf{w}^*=\argmin_{\mathbf{w}}\mathbb{E}_{D^{new}\sim\mathcal{D}^{new}}[\ell(f(\mathbf{x}_i), y_i)]$ is given by $L(\mathbf{w}^*)$ . Suppose in one particular iteration, we start with a current model $\mb{w}_g$ and after update obtain a model $\mb{w}_{new}$. Also let us write the expected error of $\mathbf{w}_g$ on the new block of data of size $2^jb_0$ as $R_{g}$ and assume that the new update block of data represents an I.I.D. sample from some underlying distribution. Note we do not assume that the entire dataset is an I.I.D. sample from the a distribution, only that the most recent block of data is.  Theorem~\ref{thm:continual_utility} demonstrates that, as can be expected, the excess error is low when $b_0$ is large, and when $R_g$ is small -- that is, when the new data set is large and when the old model is a good fit. The error increases additively as $1/\eps$.

\begin{thm}\label{thm:continual_utility}
Suppose the loss function used in Algorithm~\ref{alg:privateERM} is $L$-Lipschitz and that $|\ell|_\infty\leq M$. $\lambda\leq\frac{1}{\norm{\mathbf{w}^*-\mathbf{w}_g}_2^2 b_0}$ with $\mathbf{w}\in\mathbb{R}^d$ and update batch data size $2^jb_0$, then with probability at least $1-e^{-\eta}$: 
\begin{align*}
    \hat{L}_{D_{new}}(\mathbf{w}_{new})-L(\mathbf{w}^*) \leq\\
    \sqrt{\frac{2\eta R_{g}}{2^jb_0}}+ \frac{ 1.5M \eta+1}{2^jb_0}+\ln\left(\frac{d}{e^{-\eta}}\right)\frac{4dL^2}{\lambda b_0 \epsilon}
\end{align*}
\end{thm}

The proof of this theorem is based their  non-private analogs in the hypothesis transfer learning literature~\cite{kuzborskij2013stability}. Theorem~\ref{thm:continual_lastutility} describes the utility of the updated model on the older data. Suppose $\mathbf{w}_g$ was obtained via empirical risk minimization on dataset $D_g \sim \mathcal{D}_g$, with empirical risk $\hat{L}_{D_g}(\mathbf{w}_{g})$. The true minimizer for this problem is denoted by $\mathbf{w}_g^*=\argmin_{\mathbf{w}\in\cal{W}}\mathbb{E}_{D\sim \mathcal{D}_g}[\ell(\mathbf{w}, (x_i,y_i))]$ with  $L(\mathbf{w}_g^*)=\mathbb{E}_{D\sim \mathcal{D}_g}[\ell(\mathbf{w}_g^*, (x_i,y_i))]$. 



\begin{thm}\label{thm:continual_lastutility}. Suppose the loss function used in Algorithm~\ref{alg:privateHTL} is $L$-lipschitz and $\lambda$-strongly convex, $\lambda\leq\frac{1}{\norm{\mathbf{w}_{new}^*-\mathbf{w}_g}^2 b_0}$ with $\mathbf{w}\in\mathbb{R}^d$. Then with probability at least $1-e^{-\eta}$
\begin{align*}
|\hat{L}_{D_g}(\mathbf{w}_{new})-L(\mathbf{w}_g^*)|-|\hat{L}_{D_g}(\mathbf{w}_{g})-L(\mathbf{w}_g^*)| \\
\leq L\norm{\mathbf{w}^*-\mathbf{w}_g}\left(\sqrt{2\left(\ln\left(\frac{d}{e^{-\eta}}\right)\frac{2dL^2}{\lambda\epsilon}+1\right)}+1\right).    
\end{align*}

\end{thm}

Theorem~\ref{thm:continual_lastutility} implies that when $\mathbf{w}_g$ is close to optimal model for the new data $\mathbf{w}_{new}^*$, then the excess empirical error on the old data using the updated model is small. Thus, when the data distribution does not change significantly, the model remains valid on the old data. In such cases, increasing the regularization weight $\lambda$ helps both models. In scenarios where the data distribution changes rapidly, a smaller $\lambda$ is appropriate, which naturally increases the error of $w_g$. However, in cases where the distribution evolves, a sliding window of source data will be more appropriate. 

%% file: sections/slidingwindow2.tex
\section{Sliding window model release}

Evolving data may drift away from its original structure, as the underlying reality and generating distribution slowly changes. In such cases, instead of using the entire cumulative data, which may no longer be relevant for up to date models, we would like to use a sliding window of sufficient size $w$ to generate a source model that can be applied to the most recent $w_0$ block. Here we use $w_0 \sim b_0$ as the unit of model learning. The challenge in achieving this effect is to update model corresponding to $w$ while maintaining bounded differential privacy. In this section, we develop an algorithm to maintain a model over a sliding window $w$. This model can then be used as the source to create a model for the most recent block as in the previous section. 
\begin{figure*}[h]\vspace*{-1mm}
\centering
\begin{tabular}{c|c}
\includegraphics[width=2.9in]{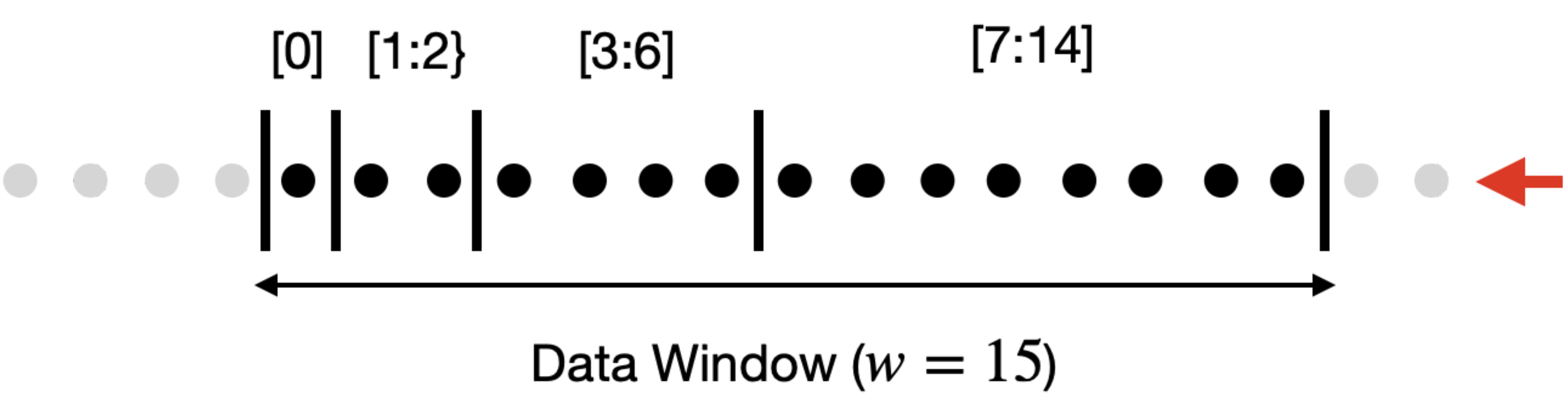} 
& \includegraphics[width=2.9in]{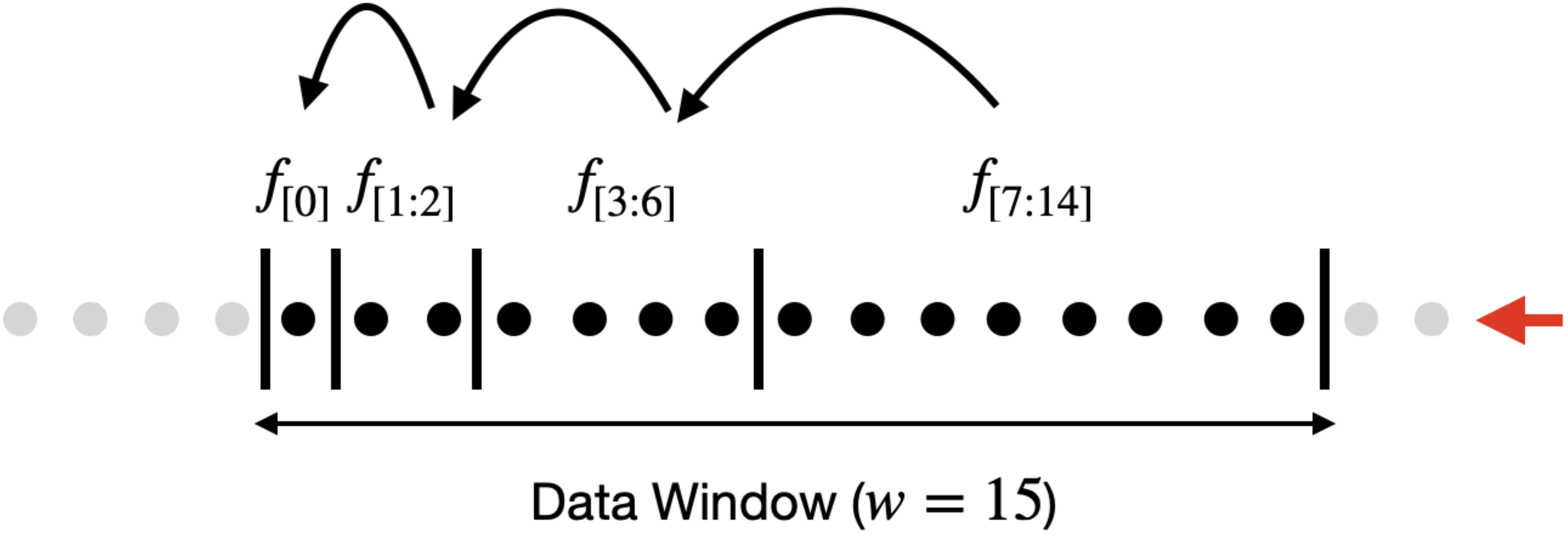} \\
(a) Sliding window size $w$, $w_0=1$. & (b) Cascade of fine tuning of models. \\
\includegraphics[width=2.9in]{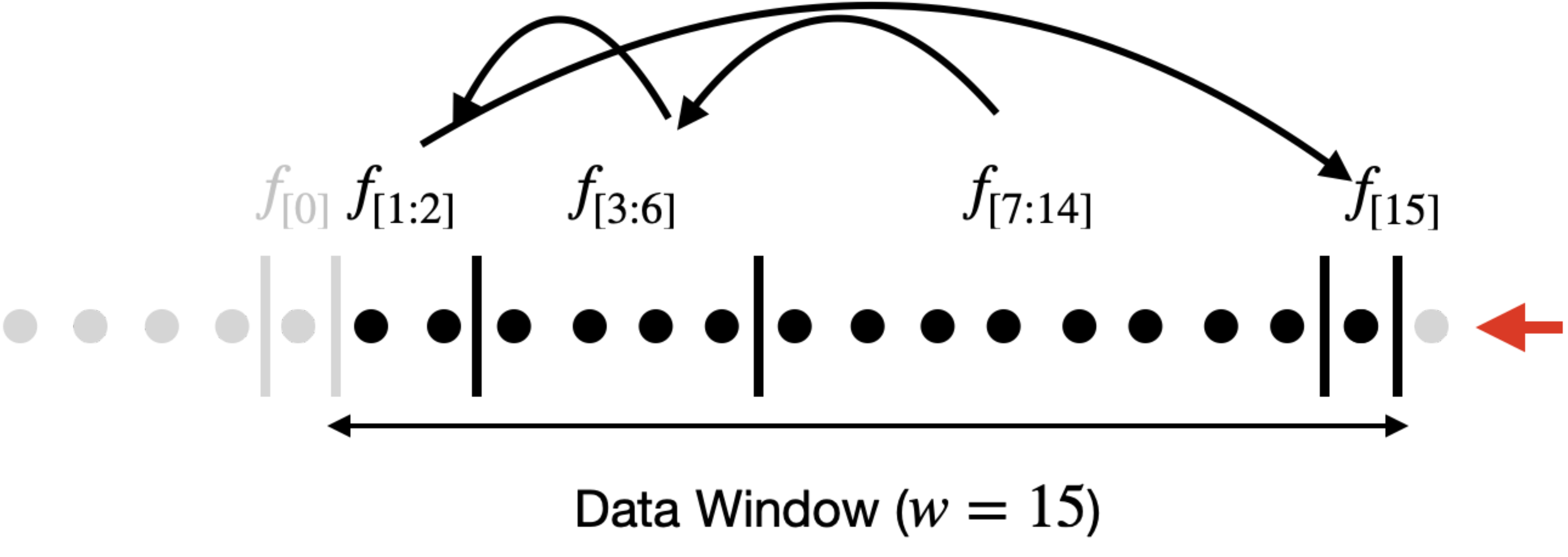}     & \includegraphics[width=2.9in]{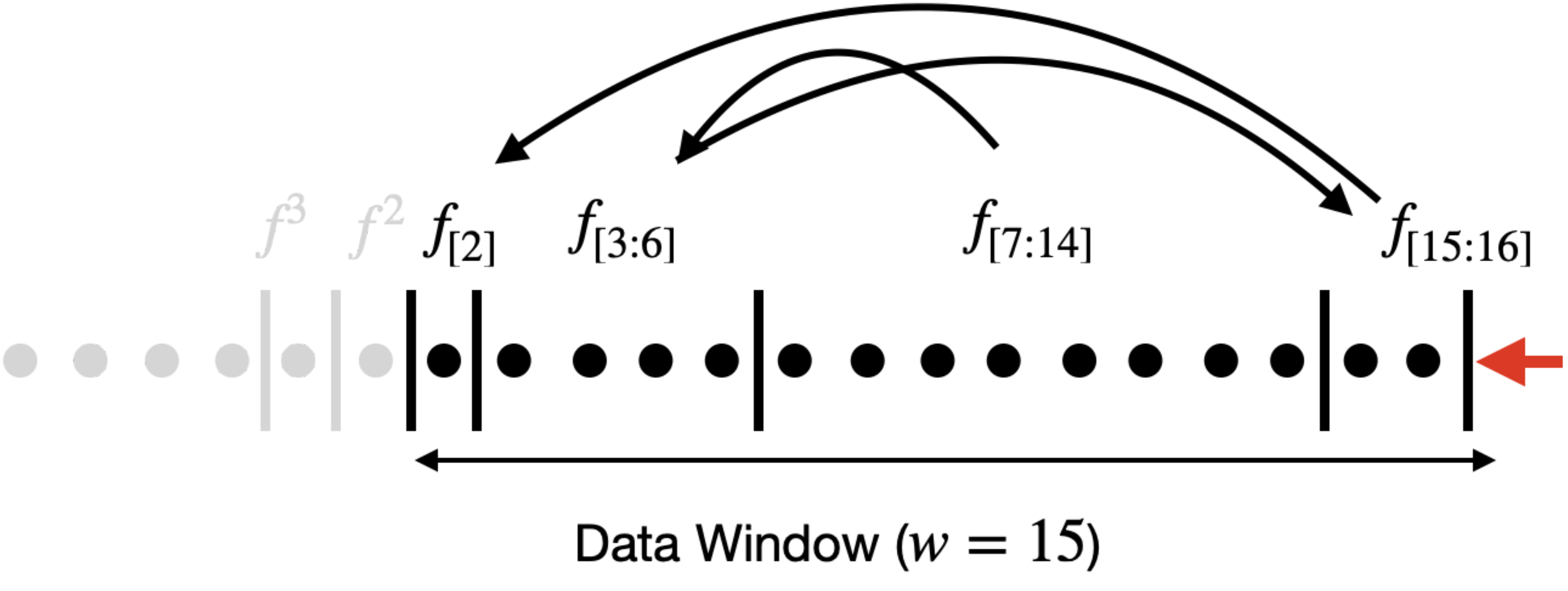} \\
(c) Sliding window update. & (d) Sliding window update. \\
\end{tabular}
\caption{Sliding window source based model release.} \label{fig:sliding} \vspace*{-1mm}
\end{figure*}
\vspace*{-2mm}
\subsection{Sliding Window Algorithm}

For simplicity of explanation, we consider a window $W$ of size $w=(\sum_{j\in[0,J]} 2^{j}w_0) - 1$, denoted by $D_{(t-w+1, t)}=\{(\mathbf{x}_i, y_i):i\in[t-w+1, t]\}$. The window is initially split into blocks of size $2^j$ for $j\in[0,J]$ as shown in Figure~\ref{fig:sliding}(a). We refer to as $f^{(k-1)}$ a model trained over a large (a constant fraction size, such as half-sized) block in $W$. The rest of $W$ is split into smaller blocks, each of size power of $2$, and models. Figure~\ref{fig:sliding}(b) shows the cascade of fine tuning of $f^{(k-1)}$ using the rest of the data of $W$. As the window progresses with time, these blocks and models are updated or reused appropriately as shown in Figure~\ref{fig:sliding}(c) and (d). When the largest block of data is no longer contained in the window, the process is refreshed to the form of Figure~\ref{fig:sliding}(b). Observe that we treat the training of new models as fine tuning of the original model, because as seen for continual release, models are close to both source and target domains when these are drawn from similar data. Algorithm~\ref{alg:slidingwindow} and Example $1$ describes the method in more detail.

\begin{algorithm}
\caption{Private Sliding Window ERM ($PSWERM$)}
\begin{algorithmic}[1]
  \scriptsize
  \STATE{Input: Dataset, $\lambda>0$, $\epsilon$, $L$, window size $w$, minimum update batch size $w_0$.}
  \STATE{/* For simplicity assume $w=2^k-1$ with $k\in \mathbb{Z}^+$ and $w_0=1$*/}
  \STATE{Partition $D_{[0: w]}$ into buckets, $w^{(i)}$ of sizes $2^{i}, 0\le i < k$}\label{line:divide}
  \STATE{Use Algorithm~\ref{alg:output} with $\Delta_{\epsilon} = \frac{6L}{\lambda \epsilon 2^{k-1}}$ on bucket $w^{(k-1)}$ to get model $f^{(k-1)}$}\label{line:initial-dependency-chain}
  \STATE{For $0\le i < k-1$, get model $f^{i}$ using bucket $w^{i}$ and regularized by $f^{i+1}$}
  \FOR{the window sliding along the stream}
    \STATE{Incorporate newly arrived data by modifying the model dependency chain as required (see Example $1$)}\label{dependency-chain-update-algo-line}
  
    \STATE{Release the model at the end of the model dependency chain}
  \ENDFOR 
\end{algorithmic}
  \label{alg:slidingwindow}
\end{algorithm}

\myparagraph{Example $1$.} Here we describe how the model dependency chains are updated for the case $w_0=1$, and $w=7$ in Algorithm~\ref{alg:slidingwindow} (also see Figure~\ref{fig:sliding} for an example with larger window $w=15$). We skip describing the method for general $w_0$ and $w$ as its involved, but it follows the same idea discussed here. Let's denote the model trained on $D_{[i:j]}$ as $f_{[i:j]}$. The base dependency chain created for the first window $D_{[0:6]}$ (Line~\ref{line:initial-dependency-chain}) is $f_{[0]} \leftarrow f_{[1:2]} \leftarrow f_{[3:6]}$. Here $f_{[i]} \leftarrow f_{[j]}$ means that $f_{[j]}$ is used as the regularizer to train $f_{[i]}$. When we receive $D_{[7]}$ (and $D_{[0]}$ goes out of window), $f_{[7]}$ is trained ($f_{[0]}$ is discarded) using the dependency chain $f_{[7]} \leftarrow f_{[1:2]} \leftarrow f_{[3:6]}$. Next when $D_{[8]}$ is received, $f_{[7:8]}$ and $f_{[2]}$ are trained and the chain becomes $f_{[2]} \leftarrow f_{[7:8]} \leftarrow f_{[3: 6]}$. Next upon receiving $D_{[9]}$ the dependency chain becomes $f_{[9]} \leftarrow f_{[7:8]} \leftarrow f_{[3: 6]}$. At the next step $D_{[3:6]}$ goes out of the window and all previous models are discarded to create new buckets and the chain becomes $f_{[4]} \leftarrow f_{[5:6]} \leftarrow f_{[7:10]}$.  

In ~\cite{datar2002maintaining}, a hierarchy is used over a sliding window for simple statistics, but allows the window to be slightly larger or smaller. In contrast, we keep the window size strictly fixed.

\myparagraph{Privacy and Utility Guarantees}
Algorithm~\ref{alg:slidingwindow} allows the release of continually updated ERM models over a sliding window of data, with only constant privacy loss. 
\begin{thm}\label{thm:sw-priv}
Algorithm~\ref{alg:slidingwindow} satisfies $\epsilon$-differential privacy.
\end{thm}
Denote the private released model for the window by $f^w$, parameterized by weights $\mathbf{w}$.  As before, the true minimizer is denoted by $\mathbf{w}^*=\argmin_{\mathbf{w}}\mathbb{E}_{D^{new}\sim\mathcal{D}^{new}}[\ell(f(\mathbf{x}_i), y_i)]$ where $D^{new}\sim\mathcal{D}^{new}$ is the last batch of data used to update the model. See the appendix for an analogue of Theorem~\ref{thm:continual_lastutility} in this scenario.

\begin{thm}\label{thm:sliding_utility }
Suppose the loss function used in Algorithm~\ref{alg:slidingwindow} is $L$-Lipschitz, $|\ell|_\infty\leq M$, and $\lambda\leq\frac{1}{\norm{\mathbf{w}^*-\mathbf{w}_{g}}_2^2 w_0}$ with $\mathbf{w}\in\mathbb{R}^d$, then with probability at least $1-e^{-\eta}$: 
\begin{align*}
    &\hat{L}_{D^{new}}(\mathbf{w})-L(\mathbf{w}^*) \\
    &\leq \sqrt{\frac{2\eta L(\mathbf{w}^*)}{w_0}}+ \frac{ 1.5M \eta+1}{w_0}+\ln(\frac{d}{e^{-\eta}})\frac{12dL^2}{\lambda w_0 \epsilon}
\end{align*}
\end{thm}

\myparagraph{Sampling based enhancements.} Sampling can be used to amplify differential privacy~\citep{balle2018privacy}. In our algorithms, sampling will reduce noise in updates with larger blocks. In the current version, all update steps use a noise scale of $O(\frac{1}{b_0})$. Sampling can be used to enhance privacy so that smaller noise scales can be used for updates that are computed over blocks larger than $b_0$. Sampling based algorithms are described in the appendix.

%% file: sections/experiments.tex
\section{Experiments}

\label{sec:experiments}

\begin{figure*}[ht!]
\centering
\begin{tabular}{ccc}
\includegraphics[width=2in]{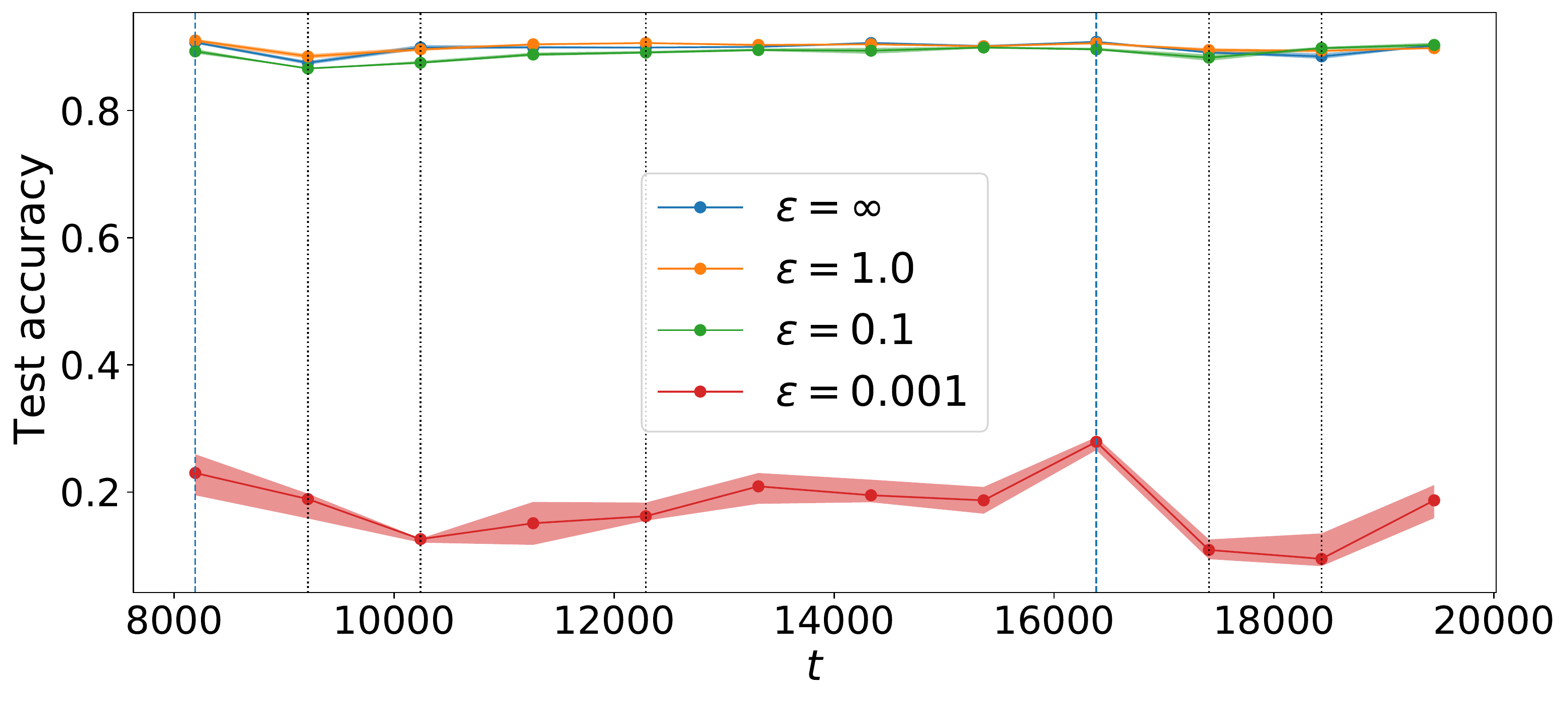} &
\includegraphics[width=2in]{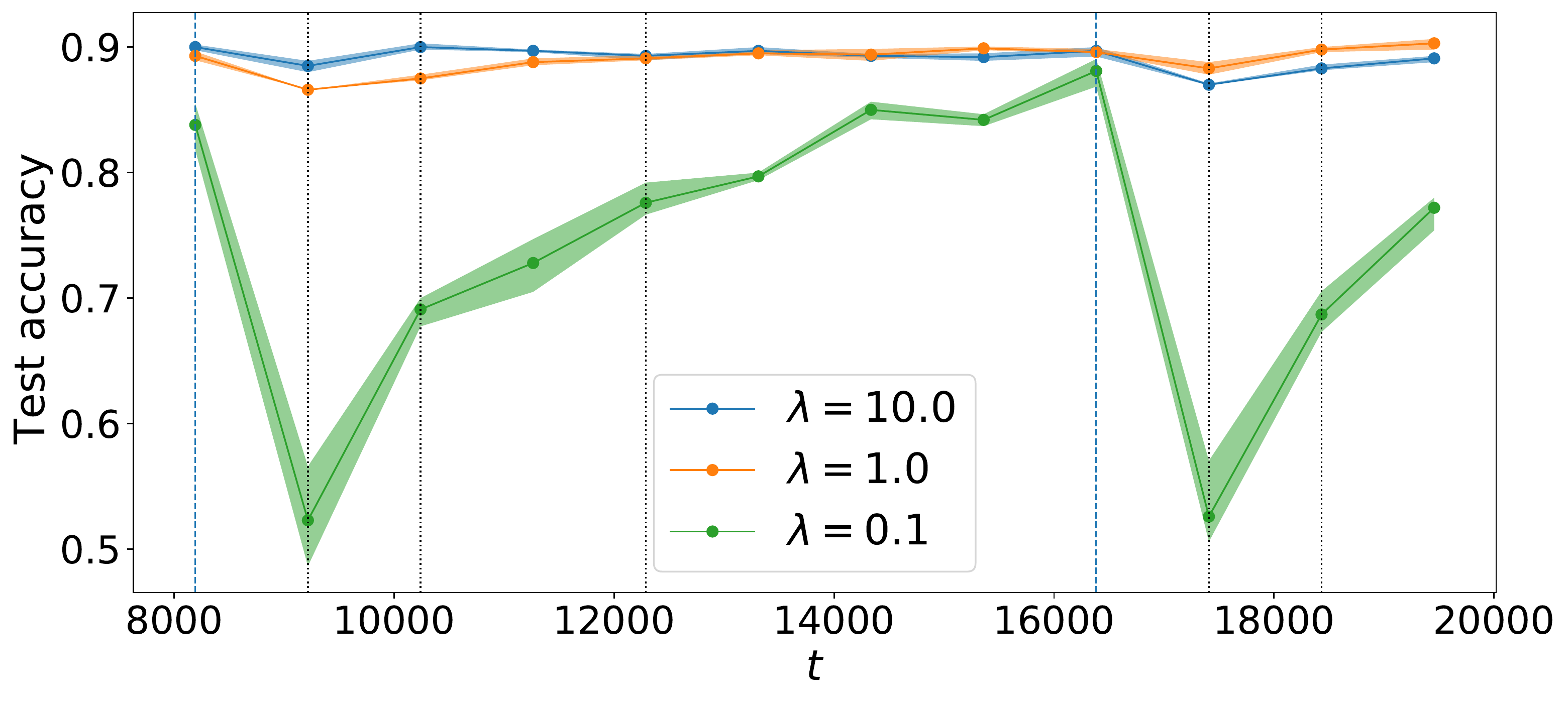} &
\includegraphics[width=2in]{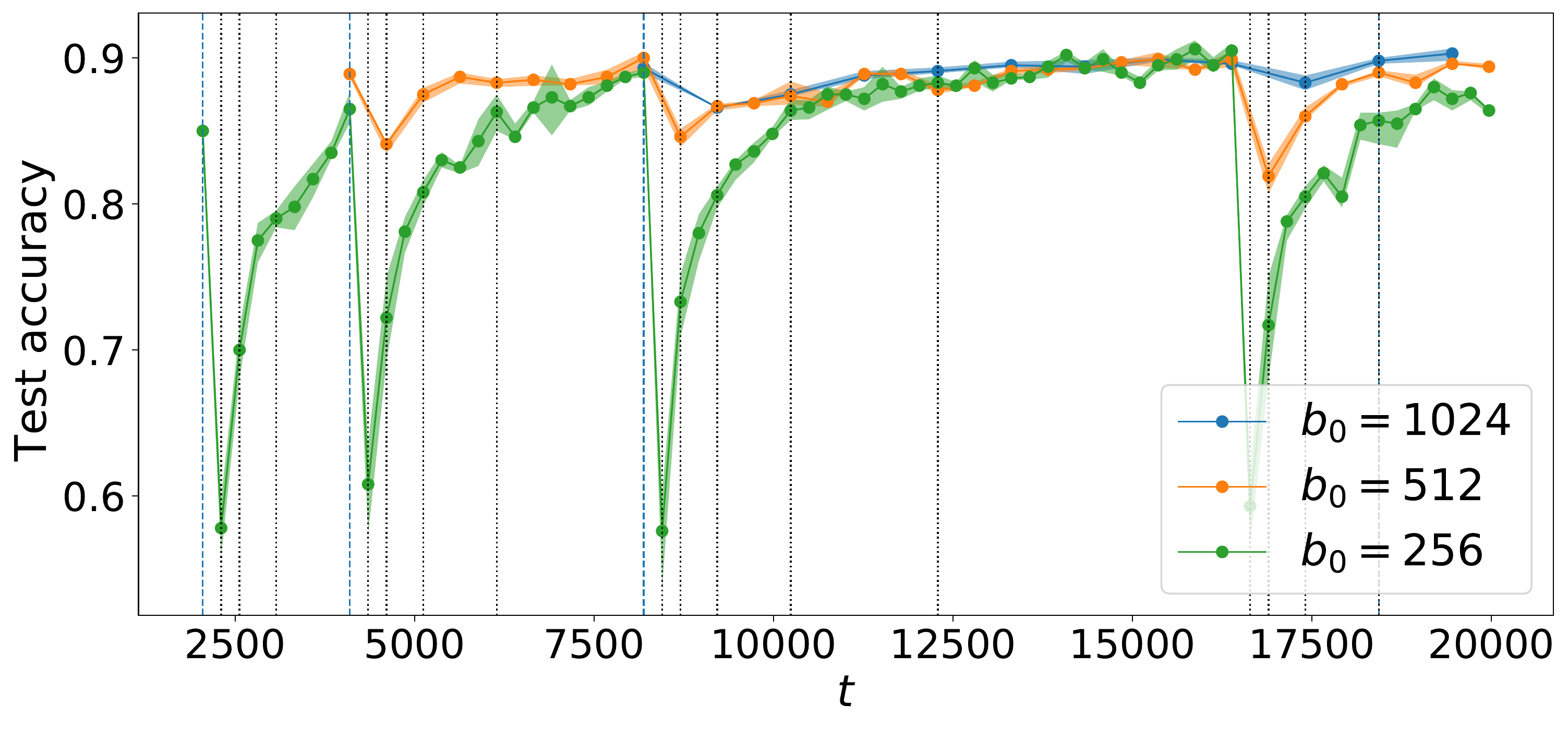}\\
(a) &  (b) & (c) \\
\vspace{-0.1cm}
\\
\includegraphics[width=2in]{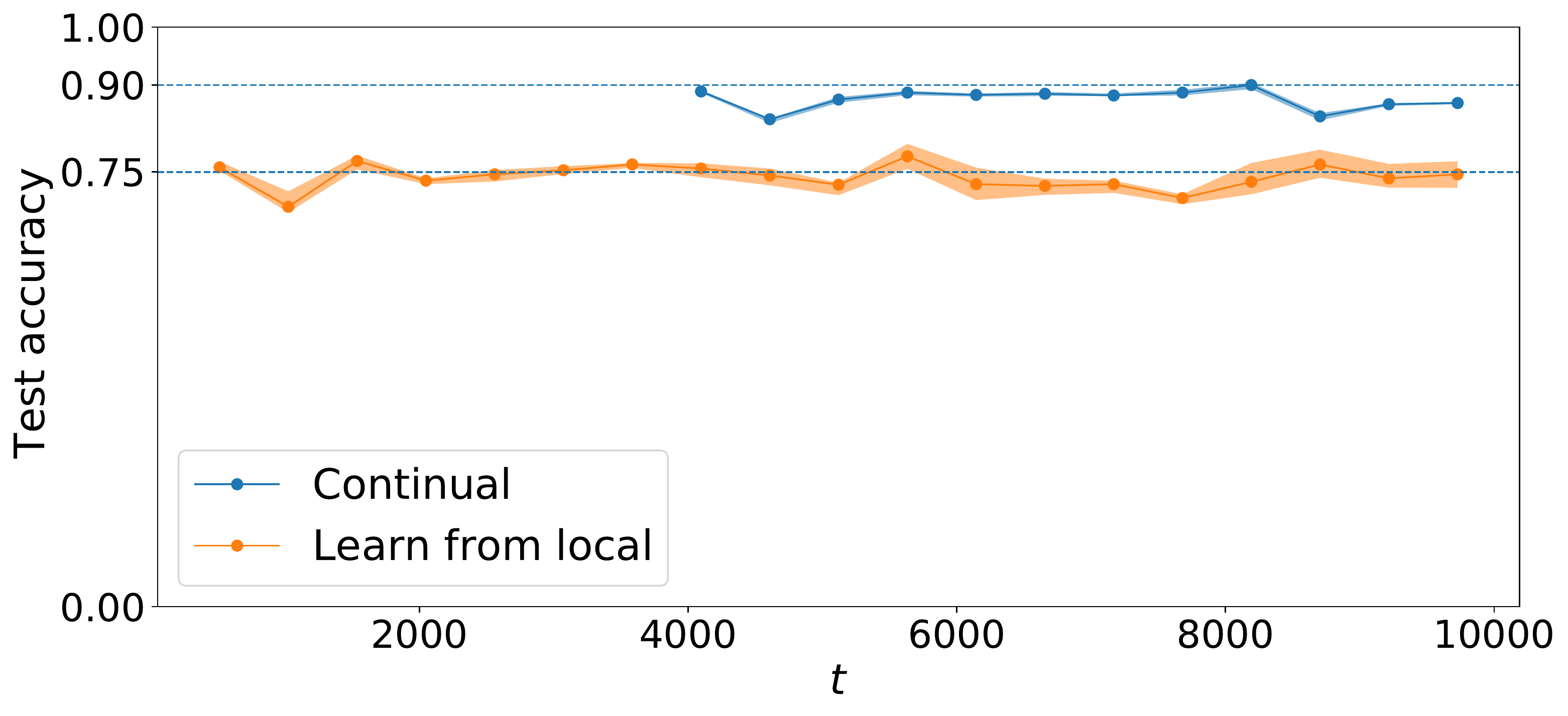} &
\includegraphics[width=2in]{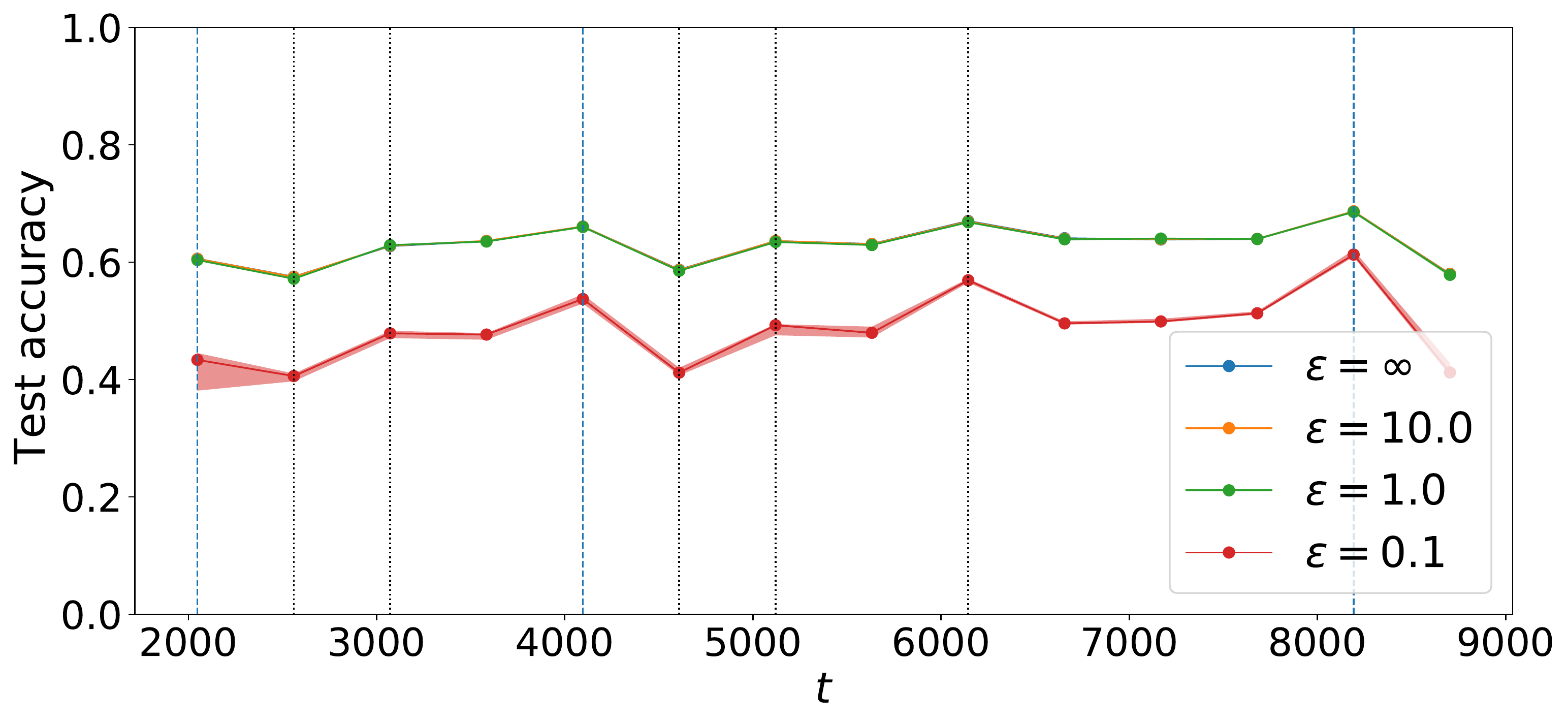} &
\includegraphics[width=2in]{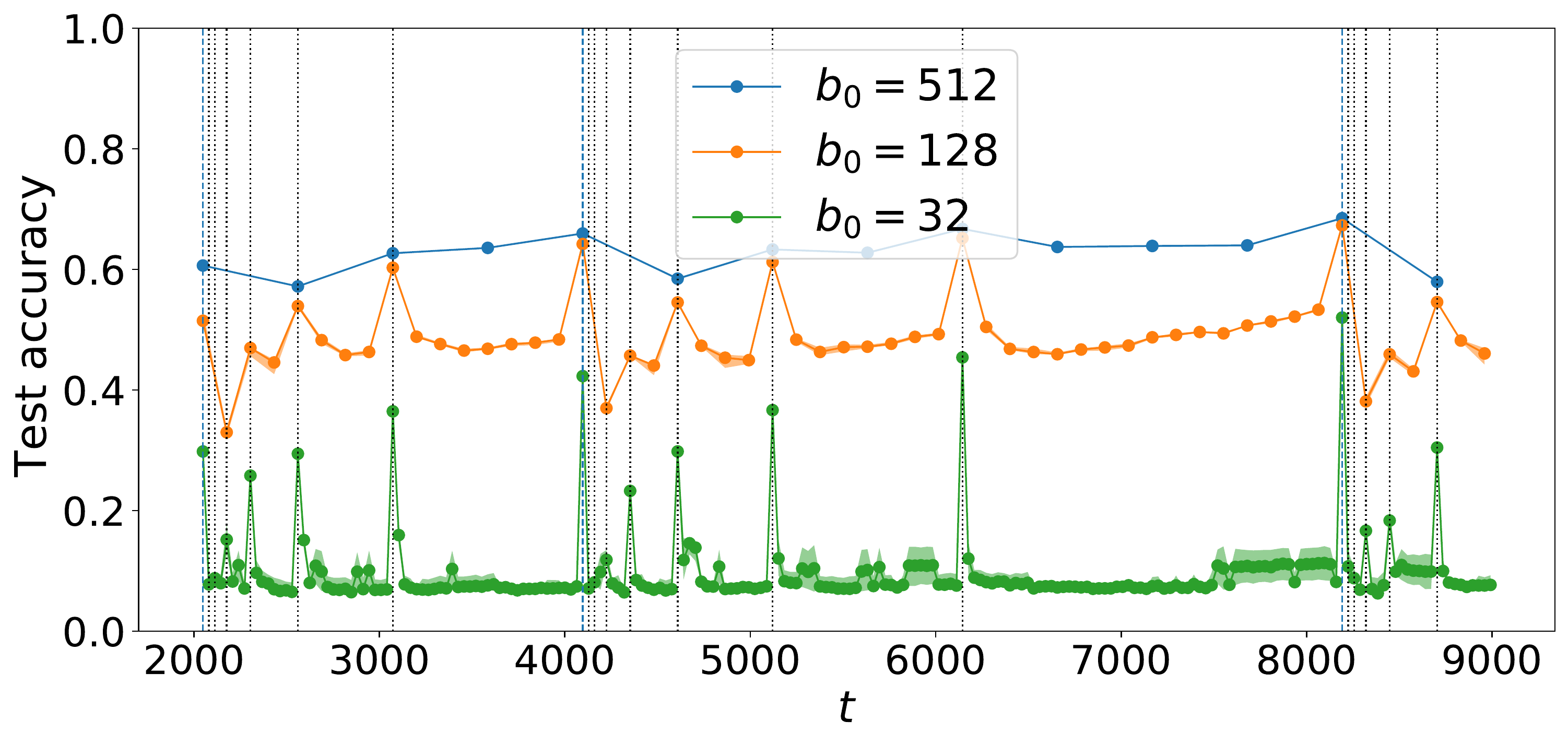}\\
(d) & (e) & (f) \\
\vspace{-0.1cm}
\\
\includegraphics[width=2in]{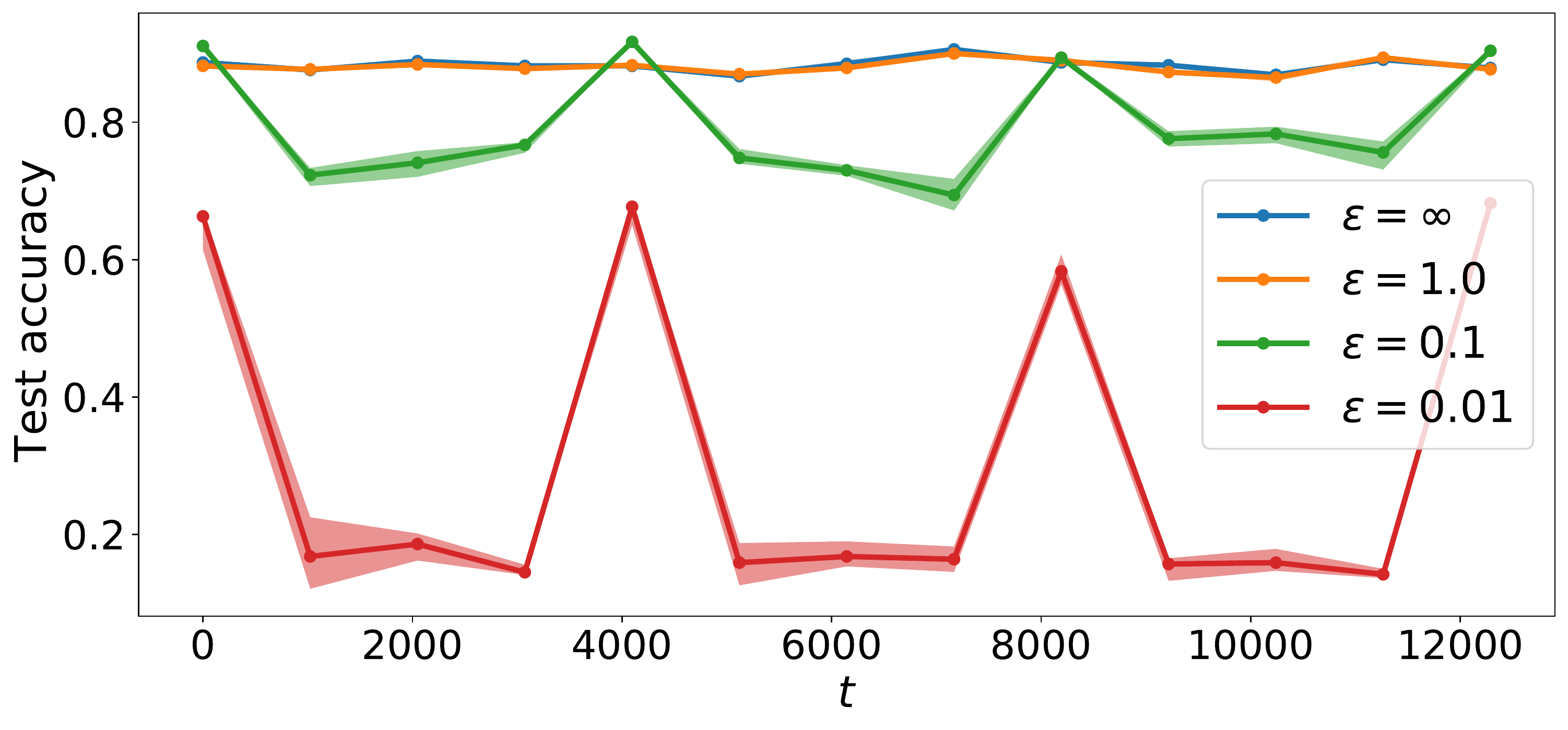} &
\includegraphics[width=2in]{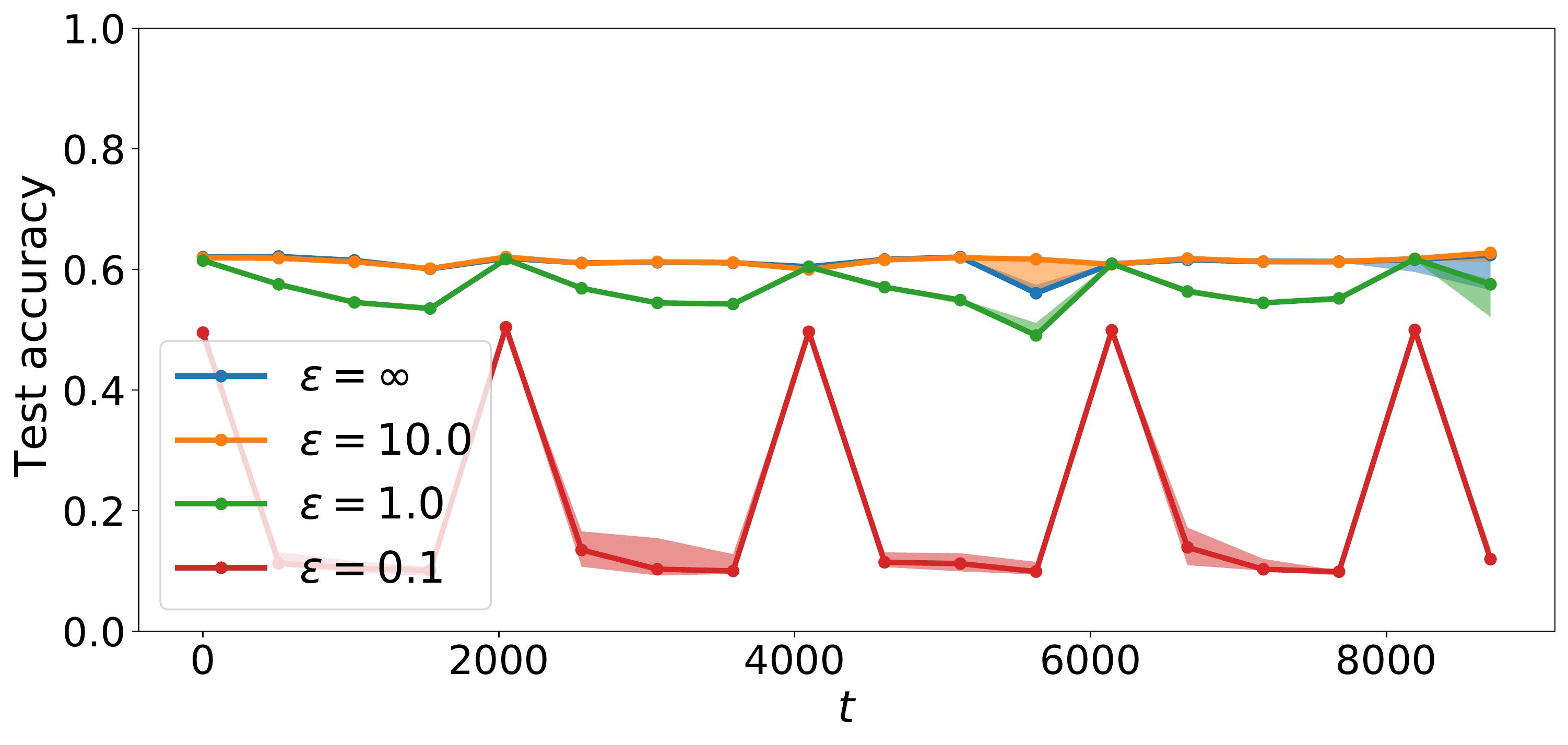} &
\includegraphics[width=2in]{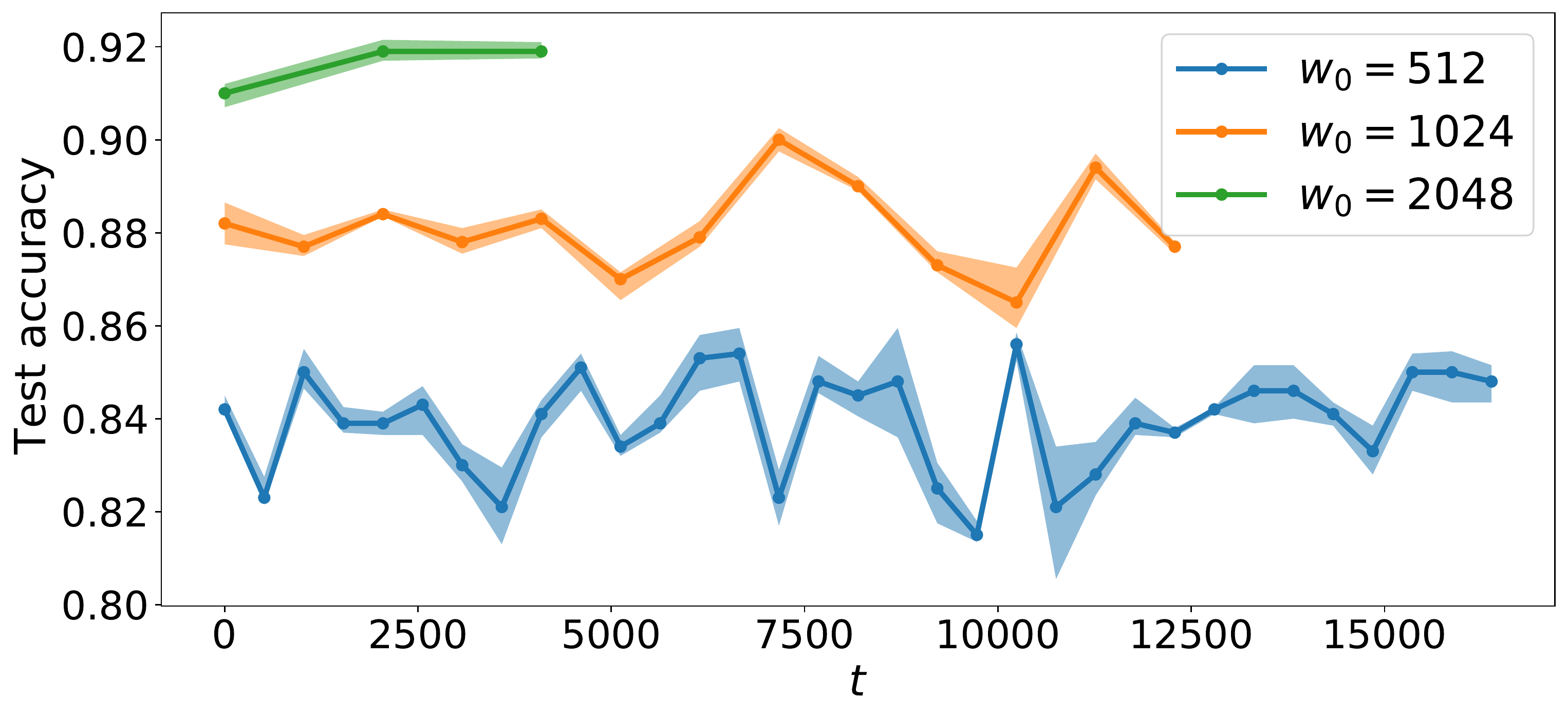} \\
(g) &  (h) & (i)\\

\end{tabular}
 \caption{Results for multi-resolution continual release algorithm on MNIST (a, b, c, d) and Arxiv dataset(e, f). \textbf{(a, b, c)} $b_0=1024$, and $B=b_0 \times 8$, $\epsilon=0.1$, and $\lambda=1$. \textbf{(d)} Compare with baseline with $b_0=512$, $\epsilon=0.1$ and $\lambda=1$. \textbf{(e, f)} For the Arxiv dataset, $B=2048$, $b_0=512$, $\lambda=10$, and $\epsilon=1$. Sliding window algorithm is evaluated in MNIST\textbf{(g, i)} and Arxiv\textbf{(i)} dataset. For MNIST $\lambda=1, \epsilon=1,$ and $w_0=1024$. Arxiv uses $\lambda=10, \epsilon=1,$ and $w_0=256$. In all cases $w=7\times w_0$. The accuracy for a sliding window is reported at the beginning of the window. For MNIST we consider sequence length of $20$k, and we don't report accuracy when the window is partially full; Thus some of the results in \textbf{(i)} are cut.}
\label{fig:continual-release-multires}
\end{figure*}

In this section we experimentally evaluate the proposed algorithms in two datasets: $(i)$ MNIST dataset of $60$k images of handwritten digits of size $28\times 28$. Here we classify the images to recognize the digits ($10$ class problem) using their $784$ pixel values. The images are randomly ordered to form a stream. The experiments only show results till $t=20$k. We use the test dataset provided with MNIST. ($ii)$ Arxiv dataset contains metadata of $10$k articles uploaded to Arxiv from $2007$\footnote{https://www.kaggle.com/Cornell-University/arxiv.} and we classify the category of an article from $17$ classes (e.g., `cond-mat', `math', etc. excluding the articles that contains multiple categories). The articles are ordered by their first submission date. We consider the vector of word counts from the title as the feature. Here we present the results with dimensions reduced to $512$ using PCA to make the base classification efficient. We held out $25\%$ data points as test and use the rest as the training sequence.

All experiments use a logistic regression classifier with SGD. Both MNIST and Arxiv dataset use minibatch of size $256$ and run for $500$ iterations. In both cases we use Algorithm~\ref{alg:output} with $\gamma=10$. All experiments are repeated $4$ times and the plots show the median in solid line with shaded area between $25$ and $75$ percentiles. All experiments were done on an intel $i7$ cpu of a laptop running Linux. The code is implemented in python using standard libraries including pytorch. The hyperparameters are not tuned for better accuracy as the primary goal of the empirical evaluation is to support the theoretical results by showing the relative performance for different privacy parameters.

The Lipschitz constant for cross entropy loss is calculated as $\frac{k-1}{2mk}\norm{X}$ where $k$ and $m$ are the number of classes and the number of training samples; $X$ is the feature matrix, and the $\norm{.}$ represents the Frobenius norm as described in~\cite{yedida2021lipschitzlr}.


\myparagraph{Evaluating the continual release algorithm}
We evaluate the continual release algorithm in Figure~\ref{fig:continual-release-multires} for the MNIST and Arxiv datasets. Both show the natural pattern -- with increasing $\epsilon$, the accuracy increases. In the MNIST dataset, the accuracy of the $\epsilon=0.1, 1$ classifiers match the non-private version. With decreasing $b_0$ and  $\lambda$, the accuracy drops.

In Figure~\ref{fig:continual-release-multires}(d) we also find that our continual release algorithm achieves better test accuracy than a baseline approach where each batch of size $b_0$ is trained and sanitized independently using Algorithm~\ref{alg:privateHTL}.


\myparagraph{Evaluating the sliding window algorithm} The sliding window algorithm is evaluated in Figure~\ref{fig:continual-release-multires}(g-i) for both the datasets. In both the datasets the accuracy increases with increasing $\epsilon$. Both the datasets achieve the non-private accuracy at $\epsilon=1$. Further with increasing the window size, $w_0$, the accuracy increases (Figure~\ref{fig:continual-release-multires}(i)).


%% file: sections/conclusion.tex
\vspace*{-1mm}
\section{Conclusion}
\vspace*{-1mm}
There remains a gap between the existing theory of differential privacy and the requirements of practical deployments, such as locally private and evolving dataset scenarios. We addressed the problem of releasing up-to-date private machine learning models for evolving datasets over infinite timescales with only constant privacy loss, focusing on the utility of our models for recently arrived data. We find that the utility of these methods depends on the available update datasize and level of regularization. The theoretical results are validated by empirical experiments.

{\bf Limitations.} The limitations of our approach include its assumptions of strong convexity and smoothness. The use of output perturbation also requires finding the exact minimizer. It will be ideal to remove these restrictions in future works. Our approach also does not achieve empirical risk minimization for the whole dataset at every step. While at exponentially growing intervals we can release a model that achieves the global ERM, the updates between these releases are only targeted at local data. The algorithm, as described, is also specific to regularized ERM via stochastic gradient descent. While in principle the concepts can work for other techniques, this question remains to be explored by future work. 

In future, in addition to investigating the limitations and extensions outlined above, we aim to develop similar algorithms for other machine learning and optimization tasks. This includes the setting of continual and sliding window release for private submodular maximization. 

%% file: sections/appendices.tex
\section{Proofs}
\subsection{Relevant Results}
The following Lemmas describe results from the private empirical risk minimization literature that are relevant to the subsequent proofs. 
\begin{lemma}[\citep{ wu2017bolt}]\label{lemma:output}
Algorithm~\ref{alg:output} satisfies $\epsilon$-differential privacy given a dataset of size $B$ and noise scale $\Delta_{\epsilon}=\frac{2L}{\lambda B \epsilon}$
\end{lemma}

\begin{lemma}[\citep{chaudhuri2011differentially, wu2017bolt}]\label{lemma:noise}
For noise vector  $\nu\sim Lap(\frac{\Delta}{\epsilon})$ of dimension $d$, with probability at least $1-\beta$,
\[\norm{\nu} \leq \ln\left(\frac{d}{\beta}\right)\frac{d\Delta}{ \epsilon}\]
\end{lemma}
\subsection{Proof of Theorem~\ref{thm:basicprivacy}}

\begin{proof}
At a high level, the privacy of this algorithm is ensured by the fact that any dataset of size $B$ is only reused as part of training sets of sizes $2B, 4B$ etc. Combining the sensitivity of regularized ERM being $O(\frac{1}{n})$ and the constant sized noise addition to each trained model, we obtain our result. 

Consider the entire dataset $D$ that is input to Algorithm~\ref{alg:multires}. Each datapoint in $(\mathbf{x}_i, y_i)\in D$ belongs to either a single dataset or no dataset of each size $s\in[B, 2B, ...., 2^kB]$. Each of these datasets of size $s$ is then used to train a model via Algorithm 1, with noise  scale $\Delta_{\epsilon}=\frac{4L}{\lambda B \epsilon}$. There is no overlap between datasets of the same size by construction, therefore each datapoint is used to train at most one model for each underlying datasize $s$.

By Lemma~\ref{lemma:output}, calling Algorithm~\ref{alg:output} with sensitivity $\Delta_{\epsilon}=\frac{4L}{\lambda B \epsilon}$ guarantees $\frac{\epsilon}{2}$-differential privacy for datasize $B$ and, more generally, $\frac{\epsilon}{2\cdot 2^k}$-differential privacy for datasize $2^kB$ or $k\in[0, \lfloor log_2(\frac{T}{B}) \rfloor]$.

 

By the sequential and parallel composition properties of differential privacy~\citep{dwork2006calibrating, mcsherry2009privacy}, the total privacy loss for any given datapoint is then given by the following geometric series: 
\begin{align*}
    \frac{\epsilon}{2}+\frac{\epsilon}{4}+... +\frac{\epsilon}{2\cdot2^{\lfloor log_2(\frac{T}{B}) \rfloor}} &= \epsilon\left(\frac{1}{2}+... +\frac{1}{2\cdot2^{\lfloor log_2(\frac{T}{B}) \rfloor}}\right) \\
    &\leq  \epsilon (1) \\
    &= \epsilon
\end{align*}

\end{proof}

\subsection{Proof of Theorem~\ref{thm:continualpriv2}}
\begin{proof}
The privacy of this Algorithm follows from the fact that each batch of data of size $b_0$ is either used once to perform a private update to an existing model, or as part of a sequence of updates using larger and larger datasets (similarly to Theorem~\ref{thm:basicprivacy}). 

Consider the disjoint batches of data of size $b_0$ over which Algorithm~\ref{alg:privateHTL} operates. For each of these batches of data, there are two usage scenarios within the Algorithm.

The first scenario is that a batch of data of size $b_0$ is used for a single update at time $t$ with datapoints $[t-b_0,t]$. The noise addition in Algorithm 3 then ensures $\frac{\epsilon}{2}$-differential privacy for this batch of data by Theorem~\ref{thm:continualpriv}. 

The second scenario is that the batch of data has been used in a sequence of updates of size $b_0, 2b_0, 4b_0, \cdots 2^jb_0$. Note, similarly to the proof of Theorem 3.1, a constant amount of noise of scale $\frac{4L}{\lambda b_0 \epsilon}$ is added regardless of underlying datasize. Therefore,  as in Theorem~\ref{thm:basicprivacy}, as the datasize increases the privacy cost of the computations decreases as follows:

\begin{align*}
    \frac{\epsilon}{2}+\frac{\epsilon}{4}+... +\frac{\epsilon}{2\cdot2^{\lfloor log_2(\frac{T}{B}) \rfloor}} &= \epsilon\left(\frac{1}{2}+... +\frac{1}{2\cdot2^{\lfloor log_2(\frac{T}{B}) \rfloor}}\right) \\
    &\leq  \epsilon (1) \\
    &= \epsilon
\end{align*}

In either case, the privacy loss is upper bounded by $\epsilon$.

\end{proof}

Note that the base model $f_b$ is not computed by Algorithm 4. Line 8 of Algorithm 4 shows that $f_b$ is obtained from the already existing output of Algorithm 2. Therefore the training procedure to obtain the model $f_b$ does not effect the privacy loss due to Algorithm 4 as the privacy guarantee of the output of a differentially private algorithm is immune to post-processing.  If we use Algorithm 4 and Algorithm 2 in sequence over the dataset, this will result in a total privacy guarantee of $2\epsilon$ by the sequential composition properties of differential privacy~\citep{mcsherry2009privacy}.

\subsection{Proof of Theorem~\ref{thm:continual_utility}}
\begin{proof}
We demonstrate the utility of this algorithm by combining the utility of the non private minimizer found via the biased regularization approach (motivated by existing work in non-private hypothesis transfer learning) with the loss in utility due to the addition of random noise.

Suppose $f^T$ is the private continual cumulative release ERM model output by Algorithm~\ref{alg:privateHTL} with corresponding weights $\mathbf{w}_T$. Note that $\mathbf{w}_T$ is obtained via the biased ERM method outlined in Algorithm~\ref{alg:privateERM} and denote the non-private weights obtained by this biased ERM procedure before the addition of random noise by $\mathbf{w}^T_{np}$. The true minimizer is denoted by $\mathbf{w}^*=\argmin_{\mathbf{w}}\mathbb{E}_{D^{new}\sim\mathcal{D}^{new}}[\ell(f(\mathbf{x}_i), y_i)]$ with corresponding expected loss $L(\mathbf{w}^*)$. We wish to bound (\ref{eq:43}).

\begin{align*}\label{eq:43}
    &\hat{L}_{D_{new}}(\mathbf{w}_T)-L(\mathbf{w}^*)\\ &=\hat{L}_{D_{new}}(\mathbf{w}_T)-\hat{L}_{D_{new}}(\mathbf{w}^T_{np})+\hat{L}_{D_{new}}(\mathbf{w}^T_{np})-L(\mathbf{w}^*)\\
\end{align*}

Due to \cite{wu2017bolt} (Lemma 11), using the Laplace Mechanism with noise scale $\nu\sim Lap(\frac{4L}{\lambda b_0\epsilon })$ to obtain $\mathbf{w}_{T}$,
\[
    \hat{L}_{D_{new}}(\mathbf{w}_T)-\hat{L}_{D_{new}}(\mathbf{w}^T_{np}) \leq L\norm \nu \\
\]
By Lemma~\ref{lemma:noise}, with probability at least $1-e^{-\eta}$,
\begin{equation}\label{eq:q1}
    \hat{L}_{D_{new}}(\mathbf{w}_T)-\hat{L}_{D_{new}}(\mathbf{w}^T_{np}) \leq \ln\left(\frac{d}{e^{-\eta}}\right)\frac{4dL^2}{\lambda b_0 \epsilon}
\end{equation}
Note that $\hat{L}_{D_{new}}(\mathbf{w}^T_{np})-L(\mathbf{w}^*) = \hat{L}_{D_{new}}(\mathbf{w}^T_{np})-\hat{L}_{D_{new}}(\mathbf{w}^*)+\hat{L}_{D_{new}}(\mathbf{w}^*)-L(\mathbf{w}^*)$. Also, by definition, as $\mathbf{w}^T_{np}$ is the \textit{empirical minimizer} we have that for base model $\mathbf{w}_g$:
\[\hat{L}_{D_{new}}(\mathbf{w}^T_{np})+\lambda\norm{\mathbf{w}^T_{np}-\mathbf{w}_g}^2\leq \hat{L}_{D_{new}}(\mathbf{w}^*)+\lambda\norm{\mathbf{w}^*-\mathbf{w}_g}^2\]
\begin{align*}\label{eq:q2}
    &\hat{L}_{D_{new}}(\mathbf{w}^T_{np})-\hat{L}_{D_{new}}(\mathbf{w}^*)\\ &\leq\lambda\norm{\mathbf{w}^*-\mathbf{w}_g}^2 - \lambda\norm{\mathbf{w}^T_{np}-\mathbf{w}_g}^2 \leq \lambda\norm{\mathbf{w}^*-\mathbf{w}_g}^2
\end{align*}
Recalling that $\lambda\leq\frac{1}{\norm{\mathbf{w}^*-\mathbf{w}_g}^2 b_0}$  we obtain:
\begin{equation}\label{eq:q3}
    \hat{L}_{D_{new}}(\mathbf{w}^T_{np})-\hat{L}_{D_{new}}(\mathbf{w}^*) \leq \frac{1}{b_0}
\end{equation}
Following~\cite{kuzborskij2013stability}, we then bound $\hat{L}_{D_{new}}(\mathbf{w}^*)-L(\mathbf{w}^*)$ using Bernstein's inequality. Let the variance be denoted by $\mathbb{V}$:
\[\mathbb{V}=\mathbb{E}\left[\sum_{i} (\ell(f^* (\mathbf{x}_i, y_i))-L(\mathbf{w}^*))^2\right]\]
By Bernstein's inequality:
\begin{align*}
&P\left( \sum_{i\in D_{new}} (\ell(f^* (\mathbf{x}_i, y_i))-L(\mathbf{w}^*))^2 > t\right) \\
&\leq exp\left(-\frac{t^2/2}{\mathbb{V}+\frac{1}{3}Mt}\right)\end{align*}
Let,
\[e^{-\eta}=exp\left(-\frac{t^2/2}{\mathbb{V}+\frac{1}{3}Mt}\right)\]
Then with probability at least $1-e^{-\eta}, \forall \eta \geq 0$, noting that $L(\mathbf{w}^*)\leq R_g$:
\begin{align}
\hat{L}_{D_{new}}(\mathbf{w}^*) &\leq L(\mathbf{w}^*)+\sqrt{\frac{2\eta\mathbb{E}[(\ell(f^* (\mathbf{x}_i, y_i))-L(\mathbf{w}^*))^2]}{m}}\\
&\hspace{1cm}+\frac{2M\eta}{3m}\\
&\leq\sqrt{\frac{2\eta L(\mathbf{w}^*)}{b_0}}+ \frac{ 1.5M \eta}{b_0}\\
&\leq \sqrt{\frac{2\eta R_g}{b_0}}+ \frac{ 1.5M \eta}{b_0}\label{eq:q4}
\end{align}

Combining (\ref{eq:q1}), (\ref{eq:q3}) and (\ref{eq:q4}) we obtain the result.

\end{proof}
\subsection{Proof of Theorem~\ref{thm:continual_lastutility}}
\begin{proof}
In order to bound the utility of the updated model on the previous batch of data, we will obtain an expression for the excess empirical risk caused by using the updated model on the previous batch of data. We consider two datasets corresponding to the  previous batch of data and the current batch of data respectively:
\begin{itemize}
    \item $D_b \sim \mathcal{D}_b$ of size $n_b$ used to train the source/base model $\mathbf{w}_b=\argmin_{\mathbf{w}} \hat{L}_{D_b}(\mathbf{w})$.
        \item $D_{new} \sim \mathcal{D}_{new}$ used to train fine-tuned model $\mathbf{w}_{new}=\argmin_{\mathbf{w}} \hat{L}_{D_{new}} (\mathbf{w})$.

\end{itemize}
Denote the true minimizers for these datasets by $\mathbf{w}^*_b=\argmin_{\mathbf{w}} \mathbb{E}_{D_b\sim\mathcal{D}_b}[L(\mathbf{w})]$ and $\mathbf{w}^*_{new}=\argmin_{\mathbf{w}} \mathbb{E}_{D_{new} \sim\mathcal{D}_{new}}[L(\mathbf{w})]$ respectively. 

We will bound the excess empirical error if the model $\mathbf{w}_{new}$ was used for a dataset $D_b \sim \mathcal{D}_b$ instead of its own empirical minimizer $\mathbf{w}_b$. 
\begin{align*}
 &|\hat{L}_{D_b}(\mathbf{w}_{new})-L(\mathbf{w}^*_b)|\\
&\leq|\hat{L}_{D_b}(\mathbf{w}_{new})-\hat{L}_{D_{b}}(\mathbf{w}_{b})|+|\hat{L}_{D_{b}}(\mathbf{w}_{b}) -L(\mathbf{w}^*_b)|
\end{align*}
\begin{align*}
 &|\hat{L}_{D_b}(\mathbf{w}_{new})-L(\mathbf{w}^*_b)|-|\hat{L}_{D_{b}}(\mathbf{w}_{b}) -L(\mathbf{w}^*_b)|\\ &\leq|\hat{L}_{D_b}(\mathbf{w}_{new})-\hat{L}_{D_{b}}(\mathbf{w}_{b})|
\end{align*}

Given that $\ell$ is $L$-Lipschitz continuous,
\begin{align}
   & \hat{L}_{D_b}(\mathbf{w}_{new})-\hat{L}_{D_{b}}(\mathbf{w}_{b}) \\
   &= \frac{1}{n_b}\sum_{i=1}^{n_b}\ell(\mathbf{w}_{new}, (\mathbf{x}_i,y_i))- \frac{1}{n_b}\sum_{i=1}^{n_b}\ell(\mathbf{w}_{b}, (\mathbf{x}_i,y_i))\\
    &= \frac{1}{n_b}\sum_{i=1}^{n_b}[\ell(\mathbf{w}_{new}, (\mathbf{x}_i,y_i))- \ell(\mathbf{w}_{b},(\mathbf{x}_i,y_i))] \\
    &\leq \frac{1}{n_b}\sum_{i=1}^{n_b} L \norm{\mathbf{w}_{new}- \mathbf{w}_{b}} \\
    &=L \norm{\mathbf{w}_{new}- \mathbf{w}_{b}} \\
    &\leq L( \norm{\mathbf{w}_{new}-\mathbf{w}^*_{new}}+\norm{\mathbf{w}^*_{new}-\mathbf{w}_{b}})\label{e0}
\end{align}
By the same reasoning as the proof of Theorem 4.3, recalling that $\ell$ is $\lambda$-strongly convex, $\lambda = \frac{1}{\norm{\mathbf{w}^*_{new}-\mathbf{w}_g}^2 b_0}$ and deriving (17) in the same manner as an identical result used in the proof of Theorem~\ref{thm:continual_utility},
\begin{align}
   & \norm{\mathbf{w}_{new}-\mathbf{w}^*_{new}} \\
   &\leq \sqrt{\frac{2}{\lambda}\norm{\hat{L}(\mathbf{w}_{new})
    -\hat{L}(\mathbf{w}^*_{new})}}\\
    &\leq \sqrt{\frac{2}{\lambda}\left[\ln\left(\frac{d}{e^{-\eta}}\right)\frac{4dL^2}{\lambda b_0 \epsilon}+\frac{1}{b_0}\right]} \\
    &\leq \sqrt{\frac{2}{\lambda b_0}\left[\ln\left(\frac{d}{e^{-\eta}}\right)\frac{4dL^2}{\lambda \epsilon}+1\right]}\\
    &\leq \sqrt{2\norm{\mathbf{w}^*_{new}-\mathbf{w}_b}^2 \left[\ln\left(\frac{d}{e^{-\eta}}\right)\frac{4dL^2}{\lambda \epsilon}+1\right]}\label{e1}
\end{align}
\end{proof}

Combining (\ref{e0}) and (\ref{e1}) we obtain the result. 
\subsection{Proof of Theorem~\ref{thm:sw-priv}}

\begin{proof}
To see the privacy of the sliding window algorithm, observe that each data batch (or data point) is used in three different ways. A data point $x$ is used to form a base model ($f^0$ in Fig.~\ref{fig:sliding}) from scratch; it is used to update/refine the current base model $f^0$ from the right; and it is used to update $f^0$ from the left. The cost in each of these phases can be bounded. The basic model formation is run once on any data point, resulting in constant privacy cost. On each side of $f^0$, a point contributes to $O(\log w)$ update computations in sets of sizes $n=w/2, w/4, w/8, \dots$ etc. Due to the ERM sensitivity bound $O(\frac{1}{\lambda n})$ (see Sec~\ref{sec:basic}), this sequence results in a constant privacy cost for the point in making updates. Below, we describe the details of this idea in reference to the description in Algorithm~\ref{alg:slidingwindow}.

Consider a time-step when the data partitioning in Algorithm~\ref{alg:slidingwindow} matches that of Figure 1 a) and b). This corresponds to the state of the algorithm after lines 3-5 in the pseudocode. 

Denote the dataset representing the window of size $w$ at this timestep by $D^w=\{\mathbf{x}_i, y_i : i\in[1,w]\}$. Let $w=\sum_{j=1}^J 2^jw_0$. We consider a single pass of the sliding window algorithm over $D^w$, before the next call of Algorithm~\ref{alg:slidingwindow} Line 5, when the  window is reset to the form shown in Figure 1 a) and b). Note that this extends the data used by the algorithm from just $D^w$ originally, to $D^{w^+}=\{\mathbf{x}_i, y_i : i\in[1,w+2^Jw_0-1]\}$ before the window is refreshed.

We consider the following three disjoints datasets for the $k^{th}$ pass of the sliding window algorithm. These datasets represent a partitioning of $D^{w^+}$ into the data before the base model ($f^0$ in Fig.~\ref{fig:sliding}), the data contained in the base model, and the data that arrives after the base model during this pass of the sliding window algorithm, before the window is refreshed:
\begin{itemize}
    \item $D_k^L=\{\mathbf{x}_i, y_i : i\in[1,(2^J-1)w_0]\}$
    \item $D_k^R=\{\mathbf{x}_i, y_i : i\in[(2^J-1)w_0+1,w]\}$
    \item $D_k^{R+}=\{\mathbf{x}_i, y_i : i\in[w+1, w+2^Jw_0-1]\}$
\end{itemize}

Note that at at the $(k-1)^{th}$ pass of the sliding window algorithm $D_{k-1}^R=D_k^L$ and $D_{k-1}^{R+}=D_k^R$. If we also consider the $(k+1)^{th}$ pass of the algorithm, then $D_{k+1}^L=D_k^R=D_{k-1}^{R+}$ and $D_{k+1}^{R}=D_k^{R+}$. Therefore over all passes of the sliding window algorithm, each of the disjoint datasets represented above will be used three times, as effectively $D_{k+1}^L, D_k^R$ and $D_{k+1}^{R+}$ for some value of $k$.

Due to this observation, the differential privacy guarantee for Algorithm~\ref{alg:slidingwindow} follows immediately if Algorithm~\ref{alg:slidingwindow} satisfies $\frac{\epsilon}{3}$-differential privacy over a single pass $k$ for each of these datasets.

Due to the privacy guarantee of Algorithm~\ref{alg:output}, outlined in the proof of Theorem~\ref{thm:basicprivacy},  Algorithm~\ref{alg:slidingwindow} satisfies $\frac{\epsilon}{3}$-differential privacy for dataset $D^R$. Only a single private empirical risk minimization algorithm with appropriately added noise is released using dataset $D^R$ before a window refresh. 

The privacy guarantee for Algorithm~\ref{alg:slidingwindow} over $D^{R+}$ and $D^L$ are identical as both call Algorithm~\ref{alg:privateERM} for models of sizes $2^0, .. 2^{J-1}$ using Laplace noise with scale $\frac{12L}{\epsilon w_0 \lambda}$. This sequence of calls satisfies $\frac{\epsilon}{3}$-differential privacy over one pass of the sliding window algorithm due to the argument previously used in Theorem~\ref{thm:basicprivacy}, at it results in privacy loss via sequential composition of:
\begin{align}
    \frac{\epsilon}{6}+\frac{\epsilon}{12}+\frac{\epsilon}{24}+\cdots &= \frac{\epsilon}{3}\left( \frac{\epsilon}{2}+\frac{\epsilon}{4}+ \frac{\epsilon}{8}+\cdots\right) \\
    \leq \frac{\epsilon}{3}(1)
\end{align}
\end{proof}

\subsection{Proof of Theorem~\ref{thm:sliding_utility }}
\begin{proof}
This result is obtained via similar reasoning to Theorem~\ref{thm:continual_utility}, recognizing that the final model output for a window is essentially the result of private regularized empirical risk minimization using the regularizer described in Equation~\ref{eq:regerm}. In the sliding window scenario, the weights $\mathbf{w}_g$ used in the regularizer correspond to the most recently updated larger block of data within the window (e.g. model $f^2$ in Figure~\ref{fig:sliding} a) or model $f^5$ in Figure~\ref{fig:sliding}) d). 

Suppose $f^w$ is the output of Algorithm~\ref{alg:slidingwindow} for a given window of data, with corresponding weights $\mathbf{w}$. Note that $\mathbf{w}$ is the result of Algorithm~\ref{alg:privateERM} using noise of scale $\frac{6L}{\lambda w_0 \epsilon}$. Denote the non-private weights found by the biased ERM method in  Algorithm~\ref{alg:privateERM} before the addition of random noise by $\mathbf{w}_{np}$. Denote the true minimizer of the objective function before the addition of random noise by $\mathbf{w}^*$, with corresponding risk $L(\mathbf{w}^*)$. $\hat{L}_D(\mathbf{w})$ denotes the empirical risk. We will bound the expanded expression outlined by (\ref{eq:52}).

\begin{align*}\label{eq:52}
    &\hat{L}_{D_{new}}(\mathbf{w})-L(\mathbf{w}^*) \\ &=\hat{L}_{D_{new}}(\mathbf{w})-\hat{L}_{D_{new}}(\mathbf{w}_{np})+\hat{L}_{D_{new}}(\mathbf{w}_{np})-L(\mathbf{w}^*) 
\end{align*}

Due to \cite{wu2017bolt} (Lemma 11), using the Laplace Mechanism with noise scale $\nu\sim Lap(\frac{12L}{\lambda w_0\epsilon })$ to obtain $\mathbf{w}$,
\[
    \hat{L}_{D_{new}}(\mathbf{w})-\hat{L}_{D_{new}}(\mathbf{w}_{np}) \leq L\norm \nu \\
\]
By Lemma~\ref{lemma:noise} with probability at least $1-e^{-\eta}$,
\begin{equation}\label{eq:p1}
    \hat{L}_{D_{new}}(\mathbf{w})-\hat{L}_{D_{new}}(\mathbf{w}_{np}) \leq \ln\left(\frac{d}{e^{-\eta}}\right)\frac{12dL^2}{\lambda w_0 \epsilon}
\end{equation}
Note that $\hat{L}_{D_{new}}(\mathbf{w}_{np})-L(\mathbf{w}^*) = \hat{L}_{D_{new}}(\mathbf{w}_{np})-\hat{L}_{D_{new}}(\mathbf{w}^*)+\hat{L}_{D_{new}}(\mathbf{w}^*)-L(\mathbf{w}^*)$. Also, by definition, as $\mathbf{w}_{np}$ is the \textit{empirical minimizer}:
\begin{align*}
&\hat{L}_{D_{new}}(\mathbf{w}_{np})+\lambda\norm{\mathbf{w}_{np}-\mathbf{w}_g}^2\\
&\leq \hat{L}_{D_{new}}(\mathbf{w}^*)+\lambda\norm{\mathbf{w}^*-\mathbf{w}_g}^2
\end{align*}

\begin{align*}\label{eq:p2}
    &\hat{L}_{D_{new}}(\mathbf{w}_{np})-\hat{L}_{D_{new}}(\mathbf{w}^*)\\ &\leq\lambda\norm{\mathbf{w}^*-\mathbf{w}_g}^2 - \lambda\norm{\mathbf{w}_{np}-\mathbf{w}_g}^2\\
    &\leq \lambda\norm{\mathbf{w}^*-\mathbf{w}_g}^2
\end{align*}
Recalling that $\lambda\leq\frac{1}{\norm{\mathbf{w}^*-\mathbf{w}_g}^2 w_0}$ we obtain:
\begin{equation}\label{eq:p3}
    \hat{L}_{D_{new}}(\mathbf{w}_{np})-\hat{L}_{D_{new}}(\mathbf{w}^*) \leq \frac{1}{w_0}
\end{equation}
Following~\cite{kuzborskij2013stability}, we then bound $\hat{L}_{D_{new}}(\mathbf{w}^*)-L(\mathbf{w}^*)$ using Bernstein's inequality in the same way as Theorem~\ref{thm:continual_utility}:
\begin{equation}\label{eq:p4}
\hat{L}_{D_{new}}(\mathbf{w}^*)-L(\mathbf{w}^*) \leq\sqrt{\frac{2\eta L(\mathbf{w}^*)}{w_0}}+ \frac{ 1.5M \eta}{b_0}
\end{equation}

Combining (\ref{eq:p1}), (\ref{eq:p3}) and (\ref{eq:p4}) we obtain the result.

\end{proof}
\section{Sampling Based Extensions of Algorithms}
Algorithm~\ref{alg:multires} releases multi-resolution models over incoming data, exploiting the fact that regularized empirical risk minimization algorithms have sensitivity  $O\left(\frac{1}{n}\right)$ to ensure only constant privacy loss over possibly infinite time-scales. However, a consequence of this approach is that the same amount of noise is added to all models trained, regardless of their underlying datasize. The addition of a  sampling procedure discussed in this section allows us to instead add less noise for larger datasizes while also maintaining constant privacy loss. We first outline this extension for multi-resolution release and then discuss how the same approach can also be applied to continual cumulative and sliding window release.

\subsection{Sampling for Multi-Resolution Release}
Suppose our goal in this section to add datasize specific noise of $\frac{4L}{\lambda 2^kB \epsilon}$ for datasize $2^kB$ in every call of the private ERM method in Algorithm~\ref{alg:output}. This change alone results in an increase of privacy cost to $\frac{\epsilon}{2}$ for every call of Algorithm~\ref{alg:output}, and therefore over time the privacy loss increases by a constant amount each time a batch of data is used. Note that each batch of data will be used up to $log_2(T/B)$ times by time-step $T$ in the original version of multi-resolution release outlined in Algorithm~\ref{alg:multires}.

Now consider the addition of the sampling step outlined in Algorithm~\ref{alg:multires_sampling} lines 4-6 to the original version of the multi-resolution release algorithm (Algorithm~\ref{alg:multires}). In this step we sample each dataset of size $2^kB$ with probability $p_0=\frac{1}{2^k}$ for each item. This results in a privacy guarantee of $p_0 \epsilon'$ for any call of Algorithm~\ref{alg:multires_sampling} line 6, if the original (without sampling) privacy guarantee for that step was $\epsilon'$~\citep{ kellaris2013practical, beimel2014bounds, balle2018privacy}. 

We then combine the sampling step and updated noise scale, and obtain a privacy guarantee of $\frac{\epsilon}{2\cdot 2^k}$ for each call of Algorithm ~\ref{alg:output} in line 6 of Algorithm~\ref{alg:multires_sampling}.  This results in a guarantee of $\epsilon$-differential privacy overall by the same argument as Theorem~\ref{thm:basicprivacy}. 

The advantage of adding this sampling procedure is that less random noise is added to the model for larger datasizes, however the disadvantage is that every dataset has an expected size of $B$, potentially impacting the accuracy of the non-private ERM solution before the noise is added. This variant of the algorithm will therefore be most useful when $B$ is large enough to guarantee strong non-private performance. 

\begin{algorithm}
\caption{Private Multi-Resolution Release with Sampling}
\begin{algorithmic}[1]
  \scriptsize
  \STATE Input: $D=\{(\mathbf{x}_t, y_t) \}$ for $t\in \mathbb{Z}^+$, $n=|D|$, block size $B$, sensitivity constant $C=\frac{4L}{\lambda B \epsilon}$ where $L$ is the Lipschitz constant.
   \FOR{$t\in \mathbb{Z}^+$} 
        \IF{$t=2^kiB$ for $k\in[0, K]$, $i\in\mathbb{Z}^+$}
        \STATE{$D=\{(\mathbf{x}_j, y_j) | j\in [t-2^kB+1, t]\}$}
        \STATE {Obtain $D_{sample}$ by sampling without replacement from $D$ with probability $\frac{\exp{\big(\frac{\epsilon }{2\cdot 2^k}\big)}-1}{\exp{\big(\frac{\epsilon}{2}\big)}-1}$.} 

            \RETURN {  $\mathbf{w} \leftarrow PSGD(D_{sample}, \Delta_{\epsilon}=\frac{4L}{\lambda 2^kB \epsilon} )$ (Algorithm~\ref{alg:output})}  
        \ENDIF
   \ENDFOR
\end{algorithmic}
  \label{alg:multires_sampling}
\end{algorithm}

\begin{theorem}\label{thm:multires-sampling-privacy}
Algorithm~\ref{alg:multires_sampling} satisfies $\epsilon$-differential privacy
\end{theorem}
\begin{proof}

For each of the datasets of size $2^kB$ for $k\in[0, \lfloor log_2(\frac{T}{B}) \rfloor]$ a model is released via the Laplace mechanism in Algorithm~\ref{alg:output}, adding Laplace Noise with sensitivity $\frac{4L}{\lambda 2^kB \epsilon}$. Denote Algorithm 1 by $A^1$. $A^1$ satisfies $\frac{\epsilon}{2}$-differential privacy for datasize $2^kB$~\citep{chaudhuri2011differentially, bassily2014private, wu2017bolt}.

Denote the composition of the sampling procedure outlined in Algorithm~\ref{alg:multires_sampling} (lines 4-6) with $A^1$ by $A^1_{samp}$. $A^1_{samp}$  satisfies $\frac{\epsilon}{2\cdot 2^k}$-differential privacy ~\citep{kellaris2013practical, beimel2014bounds, balle2018privacy}.


As outlined in the proof of Theorem~\ref{thm:basicprivacy}, due to the composition properties of differential privacy~\citep{dwork2006calibrating, mcsherry2009privacy}, the total privacy loss for each block of size $B$ is then given by $\frac{\epsilon}{2}+\frac{\epsilon}{4}+... +\frac{\epsilon}{2\cdot2^{\lfloor log_2(\frac{T}{B}) \rfloor}}$. By the properties of geometric series we obtain: 
\begin{align*}
    &\frac{\epsilon}{2}+\frac{\epsilon}{4}+... +\frac{\epsilon}{2\cdot2^{\lfloor log_2(\frac{T}{B}) \rfloor}} \\
    &= \epsilon\left(\frac{1}{2}+\frac{1}{4}+... +\frac{1}{2\cdot2^{\lfloor log_2(\frac{T}{B}) \rfloor}}\right) \\
    &\leq  \epsilon (1) \\
    &= \epsilon
\end{align*}

Therefore, Algorithm~\ref{alg:multires_sampling} satisfies $\epsilon$-differential privacy.

\end{proof}
Theorem~\ref{thm:basicutility_sampling} outlines the utility of this approach for a given datasize, derived by directly combining Theorem~\ref{thm:basicutility} with the updated noise scale.. Note that $\mathbb{E}[n']=B$.
\begin{theorem}[\textbf{\cite{wu2017bolt}}] . 
Consider 1-pass private stochastic gradient descent via Algorithm~\ref{alg:output} for $\beta$-smooth and $L$-Lipschitz loss function $\ell$ with $\Delta_{\epsilon}=\frac{4L}{\lambda 2^k B \epsilon}$. Suppose $sup_{\mathbf{w}\in\mathcal{W}}\|\ell'(\mathbf{w})\|\leq G$, $\norm{\mathbf{x}_1}\leq1$, $\mathcal{W}$ has diameter $R$ and $\mathbf{w}\in\mathbb{R}^d$. Then, for an original dataset of size $2^kB$ with $k\in[0, \lfloor log_2(\frac{T}{B})\rfloor]$ and a sampled dataset of size $n'$:
\begin{align*}
   & \mathbb{E}[\hat{L}_D(\mathbf{w}_{priv})-\hat{L}_D(\mathbf{w}^*)]\\
   &\leq \frac{((L+\beta R^2)+G^2)\log(n')}{\lambda n'}+\frac{4dG^2}{\epsilon \lambda 2^kB}
\end{align*}
\label{thm:basicutility_sampling}
\end{theorem}

\subsection{Sampling for Continual Release}
We can add a similar sampling step to the Continual Release method when the `larger update steps' in lines 10-11 of Algorithm~\ref{alg:privateHTL} are called. 
\begin{algorithm}[H]
\caption{Private Biased Regularized ERM (PBERM)}
\begin{algorithmic}[1]
  \scriptsize
  \STATE Input: Model to update $f_g$ parameterized by weights $\mathbf{w}_g$, new data $D$ of size $2^jb_0$ , regularization constant $\lambda>0$, privacy parameter $\epsilon$, Lipschitz constant $L$.
    \STATE{$\mathbf{w}^T=\argmin_{\mathbf{w}\in\mathcal{W}} \frac{1}{n}\sum_{i\in[1,n]} \ell(f(\mathbf{x}_i),y_i) + \lambda \norm{\mathbf{w} - \mathbf{w}_g}^2$}
    \RETURN{$\mathbf{w}^T+\nu$ where $\nu\sim Lap\left(\frac{4L}{\lambda 2^jb_0 \epsilon}\right)$}
\end{algorithmic}
  \label{alg:pb-erm-sampling}
\end{algorithm}

\begin{algorithm}[H]
\caption{Private Continual Release}
\begin{algorithmic}[1]
  \scriptsize
  \STATE Input: Dataset, regularization constant $\lambda>0$, privacy parameter $\epsilon$, Lipschitz constant $L$, update batch size $b_0$, block size $B$.
  \STATE{Initialize two models: $f_{b}$ and $f_l$}
  \STATE{Initialize $I_b = 0$}
  \FOR{$t\in \{B, B+1, ...\}$}
        \STATE{/* Execute the first case that matches. */}
        \CASE {$t$ is of the form $2^{i}B$ for $i\in \mathbb{Z}^{+}$}
            \STATE{Delete the existing models $f_b$ and $f_l$.}
            \STATE{Obtain the model $f_b$ for the data in the range $[0:t]$ using the previous output from Algorithm~\ref{alg:multires}.}
            \STATE{Set $f_l = f_b$ and $I_b = t$}
        \ENDCASE
        \CASE{$t$ is of the form $2^{j}b_0$ for $j\in \mathbb{Z}^{+}$}
                \STATE {Obtain $D_{sample}$ by sampling without replacement from data $[I_b: t]$ with probability $\frac{\exp{\big(\frac{\epsilon }{2\cdot 2^j}\big)}-1}{\exp{\big(\frac{\epsilon}{2}\big)}-1}$.} 

            \STATE{Learn $f_l$ with the data $D_{sample}$  and regularize with model $f_b$ using Algorithm~\ref{alg:pb-erm-sampling}}.
        \ENDCASE
        \CASE{$t$ is of the form $kb_0$ for $k\in \mathbb{Z}^{+}$}
            \STATE{Learn a model with data $[t-b_0: t]$ and regularized with model $f_l$ using Algorithm~\ref{alg:pb-erm}.}
        \ENDCASE
        \STATE{The last trained model is to be used for inference.}
  \ENDFOR
\end{algorithmic}
  \label{alg:privateHTL-sampling}
\end{algorithm}
\begin{theorem}\label{thm:cont-sampling-privacy}
Algorithm~\ref{alg:privateHTL-sampling} satisfies $\epsilon$-differential privacy.
\end{theorem}
Theorem~\ref{thm:cont-sampling-privacy} follows by the same reasoning as Theorem~\ref{thm:continualpriv2} combined with the fact that, as seen in the proof of Theorem~\ref{thm:multires-sampling-privacy}, any call of lines 11-12 satisfies $\frac{\epsilon}{2\cdot 2^j}$ differential privacy. The utility results will be equivalent to the original version of the continual release algorithm, except that for the `larger update steps' the bounds in Theorem~\ref{thm:continual_utility} using sampled datasize $n'$ will be given by (\ref{eq:c2}) for $\lambda\leq\frac{1}{\norm{\mathbf{w}^*-\mathbf{w}_g}^2 n'}$.
\begin{align*}\label{eq:c2}
    & \hat{L}_{D_{new}}(\mathbf{w}_{new})-L(\mathbf{w}^*) \\
    &\leq \sqrt{\frac{2\eta R_{g}}{n'}}+ \frac{ 1.5M \eta+1}{n'}+\ln\left(\frac{d}{e^{-\eta}}\right)\frac{4dL^2}{\lambda 2^jb_0 \epsilon}
\end{align*}
This updated bound is obtained by directly substituting the new datasize and noise scale into the existing proof. Observing that $\frac{1}{2^jb_0}\leq \frac{1}{n'}$ we can see that Algorithm~\ref{alg:privateHTL-sampling} also satisfies Theorem~\ref{thm:continual_lastutility} for $\lambda=\frac{1}{\norm{\mathbf{w}^*-\mathbf{w}_g}^2 n'}$.

\subsection{Sampling for Sliding Window Release}

The sampling adaptation can also be applied to the sliding window algorithm. In this case, there are two options for how to introduce sampling. The first option is to only alter (introduce sampling on) calls to the update steps operating via biased ERM update step via Algorithm~\ref{alg:pb-erm}. Sampled databases obtained with sampling probability $\frac{1}{2^j}$ would be used during the fine-tuning steps over the dependency chain with noise scale $\frac{4L}{\lambda 2^j b_0 \epsilon}$. The second option is to also alter every call to the base model ($f^0$ in Fig.~\ref{fig:sliding}) computed via Algorithm 1, with similar adaptations. 

\subsubsection{Fine-Tuning Sampling}
In this scenario, we replace Algorithm~\ref{alg:slidingwindow} Lines 5, 14 and 15 (updates to the base model) with the same sampling adjustments as the multi-resolution and continual cumulative release algorithms:
\begin{itemize}
    \item We sample the dataset wihtout replacement for each model update of size $2^jw_0$ using sampling probability $\frac{1}{2^j}$.
    \item We call Algorithm 7 instead of Algorithm 3 to perform the regularized ERM, using noise proportional to $\frac{1}{2^jw_0}$ instead of $\frac{1}{w_0}$. 
\end{itemize}
The proof of Theorem~\ref{thm:sw-priv} can be directly extended to demonstrate the privacy properties of sliding window release with the sampling adjustments listed above. Note that, as in the proof of Theorem~\ref{thm:multires-sampling-privacy}, these two adjustments used together result in the same sequence of privacy loss values as the original version of the sliding window algorithm for $D^{R+}$ and $D^L$ :
\begin{align}
    \frac{\epsilon}{6}+\frac{\epsilon}{12}+\frac{\epsilon}{24}+\cdots &= \frac{\epsilon}{3}\left( \frac{\epsilon}{2}+\frac{\epsilon}{4}+ \frac{\epsilon}{8}+\cdots\right) \\
    \leq \frac{\epsilon}{3}(1)
\end{align}
\subsection{Further Notes on Sampling}
The above examples demonstrate that existing results on amplification by sampling can be directly incorporated into our techniques. Specifically, we considered the method of sampling without replacement. Other methods such as amplification via Poisson sub-sampling~\cite{li2012sampling} can be included in the same way.  

\section{Experiment on basic continual release}
We ran experiments on a {\em basic continual release method}, which trains a new model on the cumulative data at each $t$ of the form $2^kB$ where $k\in \mathbb{Z}^+$. Further, at other continual release times, $t$, we release a model trained on the data at $[t-b_0 : t]$ and regularized using the model from $[t-b_0]$ using Algorithm~\ref{alg:privateERM}. In other words, it is the same as the continual cumulative release algorithm without the larger update steps. 

Figure~\ref{fig:continual-release-basic} evaluates the above algorithm for different values of $\epsilon$ and $\lambda$ for the MNIST dataset. The accuracy of the private classifier with $\epsilon\ge 0.1$ matches the accuracy of the non-private classifier. With increasing $\lambda$ the accuracy increases.

\begin{figure}[!ht]
\centering
\includegraphics[width=2.5in]{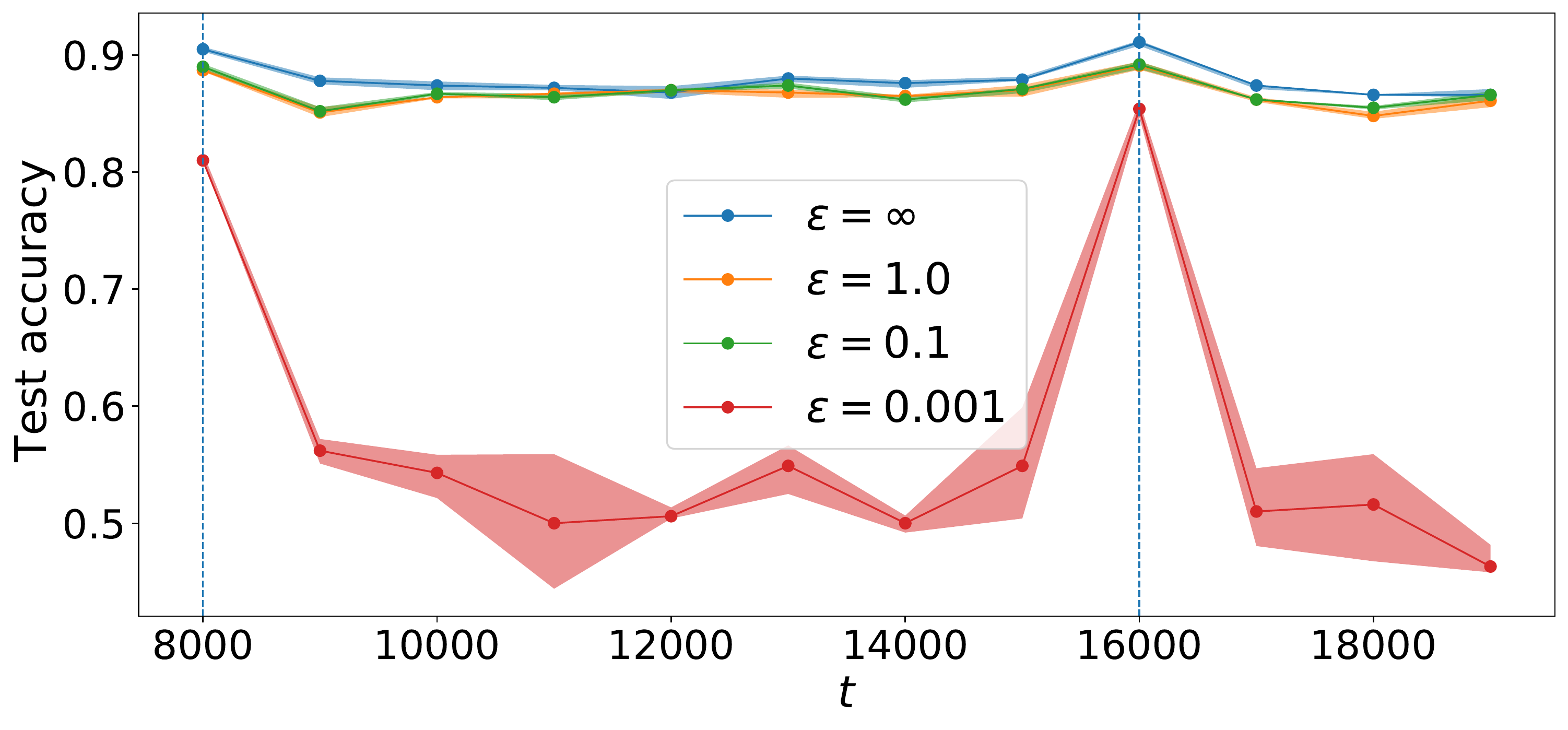}\\
(a)\\
\includegraphics[width=2.5in]{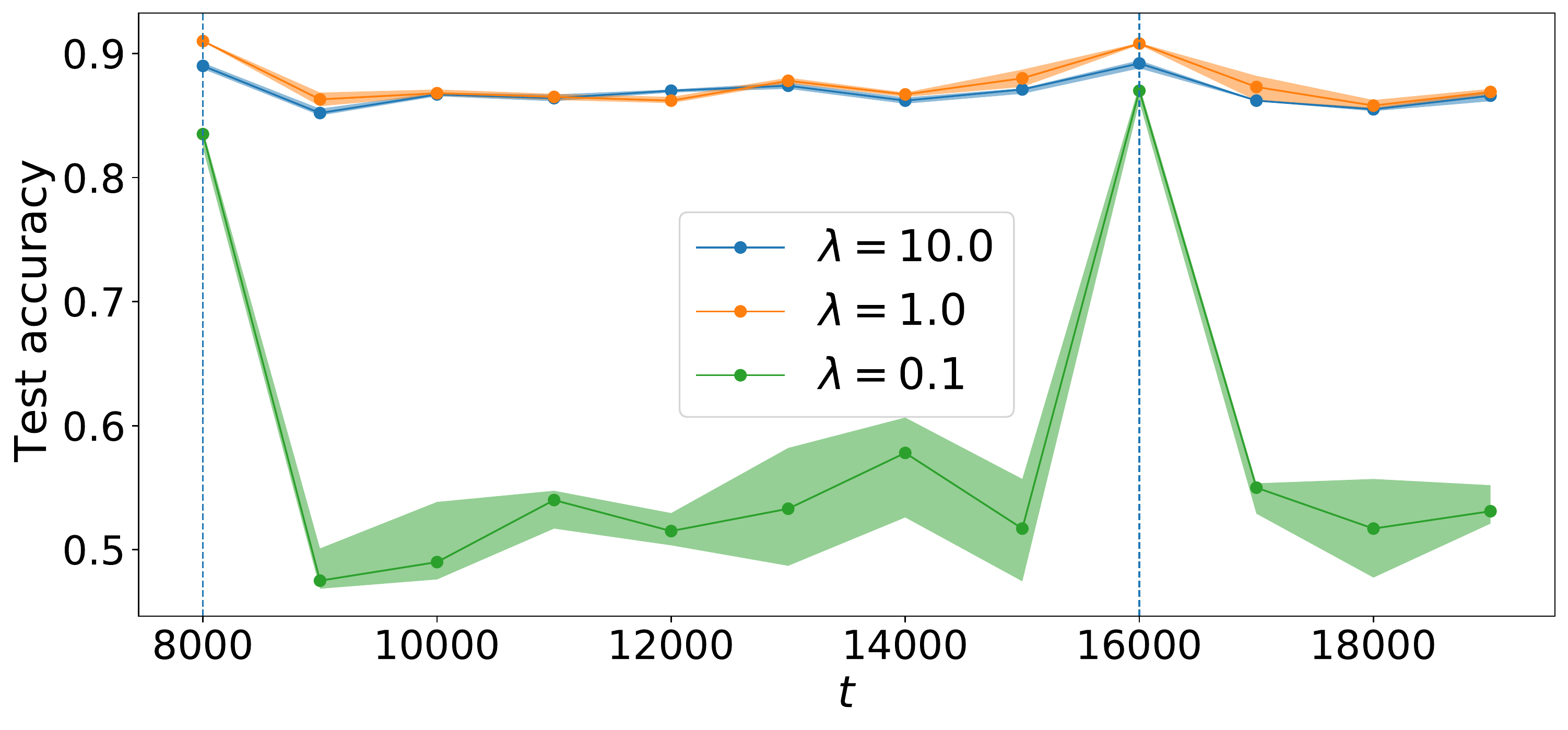}\\
(b) \\
 \caption{Results using MNIST dataset applying the basic continual release algorithm varying $\epsilon$ and $\lambda$. In (a) $\lambda=10$ and in (b) $\epsilon=1$. The $B$ and $b_0$ are set to $1000$ and $8 \times b_0$ in both the cases.}
\label{fig:continual-release-basic}
\end{figure}